\documentclass[a4paper]{scrarticle}
\usepackage[utf8]{inputenc}
\usepackage[english]{babel}
\usepackage{multirow}
\usepackage{amsmath,amssymb,amsthm}
\usepackage{microtype}
\usepackage{xcolor}
\usepackage{tikz}
\usepackage{caption}
\usepackage{subcaption}
\usepackage{rotating}
\usepackage[hidelinks]{hyperref}
\usepackage{bm}
\usetikzlibrary{arrows}
\usetikzlibrary{patterns}
\usetikzlibrary{positioning}
\usepackage{graphics}
\usepackage{nameref}
\usepackage{authblk}
\usepackage{algorithm}
\usepackage[indLines=false]{algpseudocodex}
\usepackage{xspace}

\definecolor{bittersweet}{rgb}{1.0, 0.44, 0.37}
\definecolor{airforceblue}{rgb}{0.36, 0.54, 0.66}
\definecolor{bostonuniversityred}{rgb}{0.8, 0.0, 0.0}
\definecolor{burgundy}{rgb}{0.65, 0.0, 0.13}
\definecolor{lightgray}{rgb}{0.9, 0.9, 0.9}

\newcommand{\reals}{{\mathbb{R}}}
\newcommand{\Uglob}{\mathcal{U}_{glob}}
\newcommand{\Uloc}{\mathcal{U}_{loc}}
\newcommand{\lamloc}{\lambda}
\newcommand{\tree}{{T}}
\newcommand{\Trees}{\mathcal{T}}
\newcommand{\pc}{\bm{c}}
\newcommand{\X}{\mathcal{X}}

\newtheorem{theorem}{Theorem}[section]
\newtheorem{definition}{Definition}[section]
\newtheorem{cor}{Corollary}[section]

\newcommand{\OneSol}{H$_1$\xspace}
\newcommand{\OneSolGlob}{H$_1^\textnormal{glob}$\xspace}
\newcommand{\OneSolLoc}{H$_1^\textnormal{loc}$\xspace}
\newcommand{\HTree}{H$_\textnormal{tree}$\xspace}
\newcommand{\HTreeGlob}{H$_\textnormal{tree}^\textnormal{glob}$\xspace}
\newcommand{\HTreeLoc}{H$_\textnormal{tree}^\textnormal{loc}$\xspace}
\newcommand{\HSol}{H$_\textnormal{sol}$\xspace}
\newcommand{\HSolGlob}{H$_\textnormal{sol}^\textnormal{glob}$\xspace}
\newcommand{\HSolLoc}{H$_\textnormal{sol}^\textnormal{loc}$\xspace}

\newcommand{\HAlt}{H$_\textnormal{alt}$\xspace}
\newcommand{\HAltGlob}{H$_\textnormal{alt}^\textnormal{glob}$\xspace}
\newcommand{\HAltLoc}{H$_\textnormal{alt}^\textnormal{loc}$\xspace}

\newcommand{\NomTree}{$\tree_{\text{nom}}$\xspace}
\newcommand{\SG}{SG\xspace}

\newcommand{\LSAP}{LSAP}
\newcounter{equationset}

\title{Towards Robust Interpretable Surrogates for Optimization}
\author[1]{Marc Goerigk}
\author[1]{Michael Hartisch}
\author[1]{Sebastian Merten\thanks{Corresponding author. Email: sebastian.merten@uni-passau.de}}\date{}
\affil[1]{Business Decisions and Data Science, University of Passau,\authorcr Dr.-Hans-Kapfinger-Str. 30, 94032 Passau, Germany}

\begin{document}

\tikzset{
    EdgeStyle/.append style = {->,red} 
}

\maketitle
\begin{abstract} 
An important factor in the practical implementation of optimization models is the acceptance by the intended users. This is influenced among other factors by the interpretability of the solution process. Decision rules that meet this requirement can be generated using the framework for inherently interpretable optimization models. In practice, there is often uncertainty about the parameters of an optimization problem. An established way to deal with this challenge is the concept of robust optimization.

The goal of our work is to combine both concepts: to create decision trees as surrogates for the optimization process that are more robust to perturbations and still inherently interpretable. For this purpose we present suitable models based on different variants to model uncertainty, and solution methods. Furthermore, the applicability of heuristic methods to perform this task is evaluated. Both approaches are compared with the existing framework for inherently interpretable optimization models.
\end{abstract}


\section{Introduction}
\subsection{Motivation}
Most optimization models and solution techniques are designed for practical application. Similarly, machine learning approaches aim for real-world use, with increasing emphasis on enhancing user acceptance and comprehension.
While explainability in artificial intelligence (AI) has become a major research focus, the development of approaches that prioritize user acceptance and comprehensibility in the realm of optimization remains relatively niche. This is somewhat surprising, because much like in AI, solutions to optimization problems are typically generated by solvers and algorithms that are not easily understood by all stakeholders. However, having insights into how and why specific solutions are chosen can significantly enhance the acceptance of those solutions. Conversely, a lack of transparency in the decision-making process can undermine trust and cause discontent, resulting in poor adoption of optimized decisions and ultimately rendering the optimization process ineffective.

To increase the comprehensible of machine learning models, much focus has been put on their explainability and interpretability.  Interpretability, unlike explainability, requires that the entire process---from instance to solution---is inherently comprehensible, rather than merely offering post-hoc justifications on why a specific solution was (not) chosen. As the interpretability of a method implies its explainability, interpretability is what we should strive for \cite{rudin2019stop}.

In optimization, several efforts have been made to enhance the explainability of solutions only recently. However, much less attention has been given to developing interpretable, and hence easily comprehensible, surrogates for the optimization process. Interpretable surrogates must not only be easily comprehensible, but also attain high quality solutions, which calls for a trade-off between these two conflicting goals. The proposed methods for surrogates of optimization problems base their process of attaining the surrogate on historical samples or observations. The common question in machine learning is, to what extend the found model (in our case the surrogate)---that performs well on historical observations---performs well on new data that has not been part of the training data. We go a step further and ask the question to what extend the surrogate performs well if disturbed data is used to feed the surrogate, while the obtained solution is evaluated on the true, undisturbed data.

In this work, we aim at identifying high-quality, interpretable surrogates that are robust to perturbations in the observation of instance parameters. We focus on decision trees as surrogates, which query instance parameters in their splits and provide solutions to the optimization problem at their leaves. The main idea is that the instance parameters guiding the path through the decision tree may not have been correctly observed. As a result, the traversal could lead to an undesirable leaf, and a solution might be implemented that is not suited to the actual parameters at hand.

\subsection{Motivating Example} \label{subs:intro:example}
In the following, we want to motivate our research question with an example. For a given optimization problem, we want to make use of the framework of \cite{goerigk2023framework} to find an interpretable surrogate in form of a decision tree for the optimization process. We want to evaluate the performance of this decision tree in the presence of disturbances and compare it to an alternative surrogate model. This alternative decision tree remains interpretable, but is constructed with robustness as a priority.

Let us consider the graph given in Figure~\ref{fig:graph}. The goal is to find the shortest path from node $s$ to node $t$. There are only two feasible solutions, the path containing the edges $e_1$ and $e_2$, and the path containing the edges $e_3$ and $e_4$. We will refer to them as path A ($s$-$1$-$t$) and path B ($s$-$2$-$t$), respectively. We assume that the costs to traverse edges in this graph can vary. Table~\ref{tab:observations} contains five historical samples for the edge costs, which we will use as training data.
\begin{figure}[htb]
\begin{minipage}[t]{\textwidth}
\begin{minipage}[t][3.7cm][t]{.48\textwidth}
 \centering
        \begin{tikzpicture}[shorten >=1pt,node distance=1cm, thick,main node/.style={circle,fill=blue!20,draw,font=\sffamily\Large\bfseries}, node/.style={circle,draw,font=\sffamily\Large\bfseries}]
            \node[node, minimum size=.8cm] (0) {$s$};
            \node[node, minimum size=.8cm] (1) [right = of 0] {$1$};
            \node[node, minimum size=.8cm] (2) [above = of 0] {$2$};
            \node[node, minimum size=.8cm] (t) [above = of 1] {$t$};
            \path[every node/.style={font=\sffamily\small}]
            (0) edge[->, color=black] node [below] {$e_1$} (1)
                edge[->, color=black] node [left] {$e_3$} (2)
            (2) edge[->, color=black] node [below] {$e_4$} (t)
            (1) edge[->, color=black] node [right] {$e_2$} (t);
        \end{tikzpicture}
\captionof{figure}{Example graph\label{fig:graph}}

\end{minipage}%
\begin{minipage}[b][3.7cm][t]{.48\textwidth}
    \centering
        \captionof{table}
{Historical observations\label{tab:observations}}
        \begin{tabular}{c|cccc}
             & $e_1$ & $e_2$ & $e_3$ & $e_4$ \\
            \hline
            $\bm{c}_1$ & 0     & 1     & 7     & 9  \\
            $\bm{c}_2$ & 1     & 5     & 3     & 10 \\
            $\bm{c}_3$ & 9     & 4     & 4     & 9  \\
            $\bm{c}_4$ & 9     & 10     & 5    & 7  \\
            $\bm{c}_5$ & 10    & 8     & 2     & 2
        \end{tabular}
\end{minipage}
\end{minipage}
\end{figure}
First, we want to find a decision tree that maps cost observations to solutions of this shortest path problem, and that is optimal with respect to the average costs across the five observations. We refer to this measure as the \textit{nominal} costs. To ensure interpretability, we limit ourselves to only consider univariate decision trees with a maximum depth of two. For the sake of simplicity, we further restrict ourselves to the use of branching thresholds on the midpoints between two observed values for the same edge, e.g., we only consider the thresholds $0.5$, $5$, and $9.5$ for edge~$e_1$.
\begin{figure}[htb]
\centering
\begin{subfigure}{.5\textwidth}
    \centering
        \begin{tikzpicture}[%
        level 1/.style={sibling distance=2cm},
        every node/.style = {line width=0.3mm, draw, minimum width=.5cm, minimum height=.75cm, anchor=north},
        edge from parent path={[line width=0.3mm] (\tikzparentnode.south) -- (\tikzchildnode.north)}]
        ]
        \node[shape=rectangle, rounded corners, minimum width=2cm] {$\hat{c}_1 \leq 5$}
        child {node[shape=circle, pattern color = bittersweet, pattern = north east lines] {A}}
        child {node[shape=circle, pattern color = airforceblue, pattern = dots] {B}};
        \end{tikzpicture}
        \vspace{2cm}
    \caption{Decision tree}
  \label{fig:sub1}
\end{subfigure}%
\begin{subfigure}{.5\textwidth}
    \centering
        \begin{tikzpicture}
        \fill [pattern = dots, pattern color=airforceblue](2.5,0) rectangle ++(2.75,5.25); 
        \fill [pattern= north east lines, pattern color=bittersweet](0,0) rectangle ++(2.5,5.25); 
        \draw[->, line width=0.3mm] (0, 0) -- (0, 5.5);
        \draw[->, line width=0.3mm] (0, 0) -- (5.5, 0);
        \node at (-0.25, 5.5) {$\hat{c}_2$};
        \node at (5.5, -0.25) {$\hat{c}_1$};
        \foreach \x in {0, 2, 4, 6, 8, 10}
        {
          \draw[line width=0.3mm] (\x/2,0) -- (\x/2,-0.2);
          \node at (\x/2,-0.5) {$\x$};
          \draw[line width=0.3mm] (0, \x/2) -- (-0.2, \x/2);
          \node at (-0.5, \x/2) {$\x$};
         }
        \draw[dashed, ->, line width=0.25mm] (2.5, 0) -- (2.5, 5.25);
        \draw[->, color=burgundy, line width=0.4mm](5,4) -- (2.5,4);
        \draw[fill,black] (5,4) circle [radius=2pt];
        \node at (5.35, 4) {$5$};
        \draw[fill,black] (4.5,5) circle [radius=2pt];
        \draw[fill,burgundy] (2.5,4) circle [radius=2pt];
        \node at (2.2, 4) {$\color{burgundy} 5'$};
        \node at (4.85, 5) {$4$};
        \draw[fill,black] (4.5,2) circle [radius=2pt];
        \node at (4.85, 2) {$3$};
        \draw[fill,black] (0.5,2.5) circle [radius=2pt];
        \node at (0.85, 2.5) {$2$};
        \draw[fill,black] (0,0.5) circle [radius=2pt];
        \node at (0.35, 0.5) {$1$};
        \end{tikzpicture}
    \caption{Feature space and perturbation}
    \label{fig:sub2}
\end{subfigure}
\caption{Decision tree optimal for the nominal case}
\label{fig:tree1}
\end{figure}
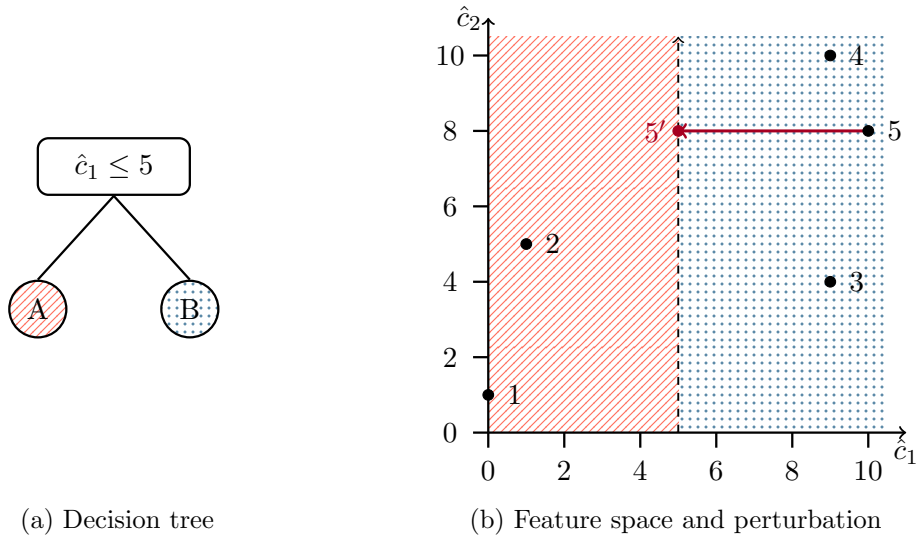
Figure~\ref{fig:tree1} illustrates the structure of an optimal decision tree for the nominal setting and its implied split in a two-dimensional projection of the feature space.
For a given cost vector $\hat{\bm{c}}\in\mathbb{R}^4$, we have to verify whether the costs of edge $e_1$ are less or equal to five. If this is the case, we will choose path A, and choose path B otherwise. It can be seen that each historical observation is assigned its optimal solution. We are able to achieve an objective value of $36$ for this nominal problem.

Now, consider that we might not be able to observe the costs of a new cost vector correctly. These perturbations will influence the way we use the decision tree to assign solutions. Note that in our setting, this interference affects the user's observation, but does not affect the actual costs of paths. In a \textit{robust} decision tree, we want to take the worst-case disturbance of the observations into account. For this example, imagine an adversary with a budget of five, which can be distributed over all historical observations to perturb an observation ($\hat{\bm{c}}_j$) of the actual costs $\bm{c}_j$. In case of the nominal tree, a worst-case perturbation occurs in cost vector $\bm{c}_5$, where the observed cost of edge $e_1$ may become $5$ and hence a user would falsely select solution A, see Figure~\ref{fig:tree1}. This results in a much worse objective value of 18 in this scenario, compared to only 4 in the nominal case, yielding overall costs of 50.

In Figure~\ref{fig:tree2}, another decision tree and its implied splits are represented. By applying it to our given historical observations, we also achieve a nominal objective value of 36. However, this tree hedges against the adversarial attack which was observed in the first setting. Hence, it is not possible to perturb the observed costs of edge $e_1$ in cost vector $\bm{c}_5$ given the budget in a way which results in an assignment of $\bm{c}_5$ to path A. A perturbation with the worst implications---increasing the observed cost of edge $e_2$ in observation $\bm{c}_2$---results in an objective value of only 41 using this tree.

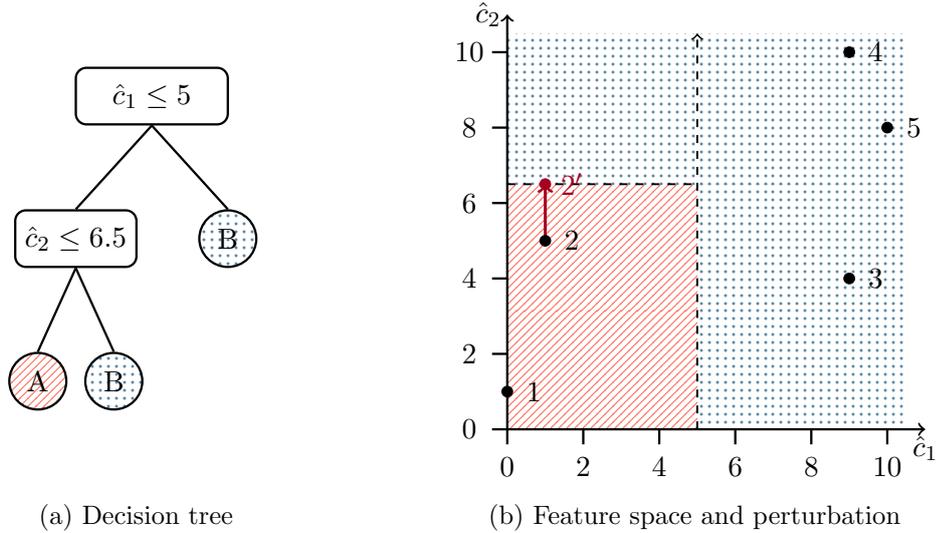
\begin{figure}[htb]
\centering
\begin{subfigure}{.5\textwidth}
    \centering
        \begin{tikzpicture}[%
        level 1/.style={sibling distance=2cm},
        level 2/.style={sibling distance=1cm},
        every node/.style = {line width=0.3mm, draw, minimum width=.5cm, minimum height=.75cm, anchor=north},
        edge from parent path={[line width=0.3mm] (\tikzparentnode.south) -- (\tikzchildnode.north)}]
        ]
        \node[shape=rectangle, rounded corners, minimum width=2cm] {$\hat{c}_1 \leq 5$}
        child {node[shape=rectangle, rounded corners] {$\hat{c}_2 \leq 6.5$}
            child {node[shape=circle, pattern color = bittersweet, pattern = north east lines] {A} }
            child {node[shape=circle, pattern color = airforceblue, pattern = dots] {B} }
        }
        child {node[line width=0.3mm, shape=circle, pattern color = airforceblue, pattern = dots] {B} };
        \end{tikzpicture}
        \vspace{1cm}
    \caption{Decision tree}
  \label{fig:sub3}
\end{subfigure}%
\begin{subfigure}{.5\textwidth}
    \centering
        \begin{tikzpicture}
        \fill [pattern = dots, pattern color=airforceblue](2.5,0) rectangle ++(2.75,3.25); 
        \fill [pattern= north east lines, pattern color=bittersweet](0,0) rectangle ++(2.5,3.25); 
        \fill [pattern= dots, pattern color=airforceblue](2.5,3.25) rectangle ++(2.75,2); 
        \fill [pattern= dots, pattern color=airforceblue](0,3.25) rectangle ++(2.5,2); 
        \draw[->, line width=0.3mm] (0, 0) -- (0, 5.5);
        \draw[->, line width=0.3mm] (0, 0) -- (5.5, 0);
        \node at (-0.25, 5.5) {$\hat{c}_2$};
        \node at (5.5, -0.25) {$\hat{c}_1$};
        \foreach \x in {0, 2, 4, 6, 8, 10}
        {
          \draw[line width=0.3mm] (\x/2,0) -- (\x/2,-0.2);
          \node at (\x/2,-0.5) {$\x$};
          \draw[line width=0.3mm] (0, \x/2) -- (-0.2, \x/2);
          \node at (-0.5, \x/2) {$\x$};
         }
        \draw[dashed, ->, line width=0.25mm] (2.5, 0) -- (2.5, 5.25);
        \draw[dashed, -, line width=0.25mm] (0, 3.25) -- (2.5, 3.25);
        \draw[fill,black] (5,4) circle [radius=2pt];
        \node at (5.35, 4) {$5$};
        \draw[fill,black] (4.5,5) circle [radius=2pt];
        \node at (4.85, 5) {$4$};
        \draw[fill,black] (4.5,2) circle [radius=2pt];
        \node at (4.85, 2) {$3$};
        \draw[->, color=burgundy, line width=0.4mm](0.5,2.5) -- (0.5,3.251);
        \draw[fill,black] (0.5,2.5) circle [radius=2pt];
        \node at (0.85, 2.5) {$2$};
        \draw[fill,burgundy] (0.5,3.251) circle [radius=2pt];
        \node at (0.85, 3.251) {$\color{burgundy} 2'$};
        \draw[fill,black] (0,0.5) circle [radius=2pt];
        \node at (0.35, 0.5) {$1$};
        \end{tikzpicture}
    \caption{Feature space and perturbation}
    \label{fig:sub4}
\end{subfigure}
\caption{Decision tree optimal for the robust case}
\label{fig:tree2}
\end{figure}

This example highlights that the performance of decision trees in an environment with uncertain parameters can significantly vary from its assumed quality when trained with respect to nominal data. It can be seen that it is possible to shrink this variance by considering the uncertainty during constructing the trees. We furthermore have seen that hedging against perturbation not necessarily comes with the price of a genuinely worse performance in the nominal setting. To find a robust decision tree, we need to take both the structure of the tree into account, and which solutions we assign to the leaves.

\subsection{Related Work} \label{subs:intro:literature}
As explainability and interpretability in optimization is still an emerging research field, the body of existing work is relatively small.
Research on explainability started out with domain-specific approaches \cite{vcyras2019argumentation} and recently has seen a surge of more general methods, e.g.\ for multi-stage stochastic optimization \cite{tierney2022explaining}, linear optimization \cite{kurtz2024counterfactual} and multi-objective optimization \cite{corrente2024explainable}. Explanations for mathematical optimization are often data-driven \cite{forel2023explainable,aigner2024framework} or use a counterfactual approach \cite{korikov2023objective,kurtz2024counterfactual}.
Most of the existing work on interpretable approaches is problem-specific, for example considering the stopping problem \cite{ciocan2022interpretable}, COVID-19 testing \cite{bastani2022interpretable}, clustering \cite{bertsimas2021interpretable,carrizosa2023clustering} and project scheduling \cite{portoleau2023robust}. Optimal policy trees \cite{amram2022optimal} and optimal prescriptive trees \cite{bertsimas2019optimal} have been proposed to obtain interpretable prescription policies, mapping instances to a predefined set of so-called treatment options.
Furthermore, the framework presented in \cite{goerigk2023framework} is suitable to describe the problem of finding inherently interpretable surrogates for a broad class of optimization problems.
In that framework, and also in the work presented in this paper, the parameters used as input for the surrogate are limited to cost parameters and the surrogate's output is a specific solutions to the optimization problem.
In \cite{goerigk2024feature}, this restriction was lifted, enabling a feature-based description of both the instance and the solution, which then serve as the input and output of the surrogate, respectively. 

In this work, we focus on using decision trees as surrogates for the optimization process. Our approach is data-driven, requiring a set of sample data points to evaluate the performance of the derived tree. Specifically, the surrogate model should perform well on average for these samples, even under worst-case perturbations of their observed parameters. This worst-case perspective is the core idea of robust optimization, which does not require known probabilities and is frequently used in theory and practice \cite{gorissen2015practical, goerigkhartisch2024rob}.
One of the key challenges when applying robust optimization is the selection of a suitable uncertainty set \cite{bertsimas2018data,chassein2019algorithms}. In this work, we assume the perturbations on the parameter observations to be restricted by variants of budgeted uncertainty \cite{bertsimas2003robust,bertsimas2004price}.

In the context of decision trees, robustness has been used for several types of classification problems in machine learning, e.g., to make them more robust against adversarial examples \cite{chen2019robust,vos2021efficient, vos2022robust} and  perturbations of the data \cite{justin2022optimal}. Other approaches aimed at improving the stability of decision trees \cite{bertsimas2023improving}.  Furthermore, in \cite{portoleau2023robust} the authors aim at finding robust decision trees for project scheduling which hedge against uncertain activity duration.
For the interested reader, we also refer to recent surveys on decision trees \cite{blockeel2023decision,costa2023recent}.

The main difference of the aforementioned works to our approach is that our surrogates do not output predefined classification labels, but solutions to optimization problems. In particular, there is an exponential number of potential outputs in our case, in contrast to a discrete, predefined set of labels. Another difference is that in case of classification, an error on the training set is minimized (e.g.\ minimize the number of misclassifications), while we, on the other hand, are dealing with cost deviations for non-optimal assigned solutions that depend on the true cost scenario as well as the provided solution. Hence, there is a connection to the predict-then-optimize framework \cite{elmachtoub2022smart}, where the criterion for the prediction process is not the accuracy of the prediction itself, but its effect on the subsequent (optimization) process that uses the prediction. In this realm, advances have been made to improve comprehensibility \cite{blanquero2023explainable}. While there are similarities in the evaluation function, it is important to note that we do not predict parameters; instead, we map parameters directly to solutions.

\subsection{Contribution and Structure of the Paper}

\paragraph{Contribution of this paper.}
Our work builds upon the framework presented in \cite{goerigk2023framework}, which introduced interpretable surrogates for optimization.  Their key idea is to generate surrogates for the optimization process, which means that we focus on finding univariate decision trees to replace the black-box optimization process. We assume that the parameters of an optimization instance are uncertain due to, e.g., lack of possibilities for measurement or inaccuracies. We investigate a robust approach that aims at finding a surrogate that is optimal with respect to the worst-case of parameter observation perturbations allowed by the specified uncertainty set. We consider two types of budgeted uncertainty sets and develop exact and heuristic solution approaches. Furthermore, we analyze the complexity of our setting and note special cases that become easier to solve. In computational experiments, we consider the impact of the choice of uncertainty set as well as the size of the uncertainty, and compare to the previous, nominal approach. We find that at a small cost in terms of nominal performance, our robust models can achieve considerably better performance on disturbed observations, and even on new observations that are taken out-of-sample.

\paragraph{Course of the work.}
In Section~\ref{sec:optApproach} we introduce the problem formulation, including the considered types of uncertainty, and present an iterative solution method for generating optimal robust surrogates in form of a decision tree. Since this method is computationally expensive, several heuristics are presented in Section~\ref{sec:heuristics}.
In Section~\ref{sec:complexity} we elaborate on the interconnection of the presented heuristics and the underlying problem, as well as on computational complexity results. We present computational experiments in Section~\ref{sec:experiments} and conclude the paper in Section~\ref{sec:conclusion}.

\section{Robust Interpretable Surrogates\label{sec:optApproach}}

\subsection{Problem Formulation and Uncertainty Sets}
\label{subs:optApproach:uncertainty}

Consider a minimization problem with a set of feasible solutions $\mathcal{X} \subseteq \mathbb{R}^n$, a linear objective function, and $N\in\mathbb{N}$ historical samples $\pc_1,\ldots, \pc_N\in\mathbb{R}^n$ of realizations of the objective function vector. Throughout the paper, we use the notation $[N]=\{1,\ldots,N\}$ to denote index sets and for simplicity of presentation, assume that the problem $\min_{\bm{x}\in\mathcal{X}} \bm{c}^\top\bm{x}$ can be represented as a mixed-integer linear program (MIP). The framework presented in \cite{goerigk2023framework} aims at finding a surrogate for the optimization process that maps realizations of the objective function vector to solutions. We investigate the case that the observed objective function values are perturbed, affecting the mapping and potentially resulting in solutions that perform poorly under the true costs.
Hence, we assume that the cost parameters of a given instance are fixed, but the observed values can be perturbed, meaning they may deviate from the true underlying values, which do not change. With this perspective, the goal is to find a robust surrogate for the optimization process in which a decision tree maps (disturbed) cost scenario observations to corresponding solutions. Here, robustness implies achieving the best average performance over all samples, even under the worst-case realization of their observed values within the chosen uncertainty set. 

Let $\Trees \subseteq\{\tree:\mathbb{R}^n \rightarrow \X \}$ denote the set of decision trees that map cost vectors to solutions. Note that for the sake of obtaining easily comprehensible surrogates, we restrict $\Trees$ to contain only  univariate trees and limit the tree's depth to a specified maximum. In particular, every split of the tree queries the cost of a single item. Furthermore, we demand that splits are based on the historical samples, i.e., we assume that a split on item $i\in[n]$ uses a threshold from the interval $[\min_{j\in[N]}c_{j,i},\max_{j\in[N]}c_{j,i}]$.

For a disturbed observation $\hat{\pc} \in \mathbb{R}^n$ corresponding to sample $\pc$, let $\tree(\hat{\pc}) \in \X$ be the solution this observation is mapped to when using decision tree $\tree\in \Trees$. Note that the resulting cost is then given by $\pc^\top  \tree(\hat{\pc})$, as only the observation and not the cost vector itself is disturbed. Hence, the problem of finding an optimal robust decision tree can be stated as
\begin{equation}
\min_{\tree \in \Trees} \max_{(\bm{\xi}_1,\ldots, \bm{\xi}_N) \in \mathcal{U}} \sum_{j \in [N]} \bm{c}_j^\top \tree(\bm{c}_j+\bm{\xi}_j)\  ,
\tag{RIT}\label{eq::TheProblem}
\end{equation} 
where $\bm{\xi}_j\in \mathbb{R}^{N}$ specifies the perturbation applied to the observations of sample $j\in [N]$, and $\mathcal{U}$ is the set of possible perturbations, which we call the uncertainty set. The abbreviation RIT stands for the ``robust interpretable tree'' problem.

In this paper, we consider a variant of so-called budgeted uncertainty \cite{bertsimas2004price}, where the deviation from the nominal costs cannot exceed a given budget.
We consider two types of budgeted uncertainty which differ in their way of allowing the budget to be spent. The set  $\Uglob$ represents the uncertainty set for a given global budget. The adversary can perturb the observations by in total $\Gamma^{\text{glob}}\in\mathbb{R}$. This setting allows large perturbations of one single observation, whilst other observations remain equal to the nominal ones.
Formally, this global budgeted uncertainty set is defined as
\begin{equation*}
    \Uglob(\Gamma^{\text{glob}}) = \{  (\bm{\xi}_1,\ldots,\bm{\xi}_N)^\top\in \reals^{N\times n} : \ \sum_{j \in [N]} \sum_{i \in[n]}  |\xi_{j,i}| \leq \Gamma^{\text{glob}} \}.
\end{equation*}
In addition to this concept of a global budget, we explore an uncertainty set $\Uloc$ that imposes a budget on the perturbation for each individual observation, which we call a local budget. This local budgeted uncertainty set is defined as
\begin{equation*}
    \Uloc(\Gamma^{\text{loc}}) = \{ (\bm{\xi}_1,\ldots,\bm{\xi}_N)^\top \in \reals^{N\times n} : \sum_{i \in [n]} |\xi_{j,i}| \leq \Gamma^{\text{loc}} \ \forall j \in [N] \}.
\end{equation*}
Observe that $\Uglob(\Gamma) \subseteq \Uloc(\Gamma)$ for any value of $\Gamma$, which means that the global budgeted uncertainty set is more restrictive than the local one, if we use the same budget. However, we obtain $\Uloc(\Gamma) \subseteq \Uglob(N\cdot\Gamma)$, i.e., the global set becomes more powerful for an appropriate scaling of the budget.

\subsection{Iterative Solution Approach\label{sec:iterative}}
We tackle Problem~\eqref{eq::TheProblem} via an iterative approach, using the well-known method of scenario generation, see e.g. \cite{aissi2009min,zeng2013solving} and \cite[Chapter~5.1]{goerigkhartisch2024rob}.
During the first step, a decision tree is generated by solving an MIP, referred to as the master problem. The model considers only a subset of perturbation scenarios from the uncertainty set. Note that finding an optimal binary decision tree is already NP-hard \cite{laurent1976constructing}.

In the second step, we solve the adversary problem. An instance of this problem takes in a decision tree (as output from the master problem) and checks the average cost resulting from a worst-case perturbation of cost observations under the considered uncertainty set. The goal is to find a perturbation that maximizes the total cost across all samples, where the solutions are determined by traversing the decision tree using the perturbed observations. In case the found objective value differs from the objective value of the master problem, the worst-case perturbation is added to the subset of perturbation scenarios and the process is repeated. The method terminates as soon as the objective value of master and adversary problem is equal.

Let us now take a closer look at the master problem. Note that the description of the master problem is independent of the considered uncertainty set. Let $\mathcal{U}'$ be a finite subset of the uncertainty set. To simplify the explanation, we can assume that $\mathcal{U}' = \{(\bm{0},\ldots,\bm{0})\}$ at the start of the algorithm, meaning that only the scenario that does not perturb the historical samples is considered, which is a valid subset for both of the considered uncertainty sets. For easier notation, let $S=\vert \mathcal{U}'\vert$ be the current number of scenarios considered.  In contrast to the model presented in \cite{goerigk2023framework}, we use the stronger flow-based formulation to model the decision tree \cite{justin2022optimal,aghaei2024strong}. To this end, we need to introduce some notation. Let $D\in \mathbb{N}$ be the fixed  depth of the binary tree. Since the depth and width of the tree is determined in advance, only the parameters for the branching nodes and the solutions in the leaves have to be determined as part of the optimization problem. The $2^{D+1}-1$ many nodes are split up into the set of inner nodes $Q$, where instance parameters are queried to perform the splits, and the set of leaves $K$ that hold solutions to the underlying optimization problem. Let $Q_-\subset Q$ be the set of inner nodes without the root node of the tree, labeled with $0$. Let $l,r:Q\rightarrow  Q\cup K$ be the functions mapping every inner node to its left and right child, respectively. Furthermore,  $p:Q_-\cup K\rightarrow Q$ maps every node (except of the root node $0$) to its predecessor in the tree. We use the vector $\bm{x}_k \in\mathbb{R}^n$ to represent the solution proposed at leaf $k\in K$. The binary variables $z^{j,s}_{q,c}$ indicate for inner node $q\in Q$ and its child $c\in\{l(q),r(q)\}$, whether for sample point $j\in [N]$ the decision path traverses node $q$ towards $c$ in perturbation scenario $s\in [S]$.

We precompute splits that are worth considering. For every $i\in[n]$, let $\Theta(i)$ be a set of possible split points in dimension $i$. Within this paper, we use the midpoints between the sorted entries, i.e.,
let $\bar{c}_{1,i}\le \bar{c}_{2,i} \le \ldots \le \bar{c}_{N,i}$ be the sorted entries in dimension $i\in[n]$, and we use
$\Theta(i)=\{\frac{1}{2}(\bar{c}_{j,i}+\bar{c}_{j+1,i}) \mid j \in [N-1]\}$.
We can then use binary variables $b_{q,i,\theta}$ to indicate whether the cost of item $i\in[n]$ is queried via $c_i \leq \theta$ for $\theta \in \Theta(i)$ at inner node $q\in Q$. Furthermore, the continuous variables $\bar{z}^{j,s}$ contain the objective value resulting from following the decision path for sample $j \in [N]$ in scenario $s\in[S]$ when using the decision at the respective leaf. Finally, the epigraph variable $u$ holds the worst-case overall cost resulting from using the found tree. Problem~\eqref{eq:mpGlob} shows the full master problem.
\begin{subequations}
\begin{align}
    \min \quad &  u \label{eq:mp:of}\\
    \text{s.t.} \quad & \bm{x}_k \in \mathcal{X} & \forall k \in K \label{eq:mpGlob:con2}\\
 &   \sum_{i \in [n]} \sum_{\theta \in \Theta(i)} b_{q,i,\theta} = 1 & \forall q \in Q \label{eq:mpGlob:con1}\\
    & z^{j,s}_{0,l(0)}+z^{j,s}_{0,r(0)} =  1 &  \forall j \in [N],\ s \in [S]  \label{eq:mpGlob:con3}\\
    & z^{j,s}_{p(q),q} = z^{j,s}_{q,l(q)} + z^{j,s}_{q,r(q)} &  \forall j \in [N],\ s \in [S],\ q \in Q_- \label{eq:mpGlob:con4}\\
    &z^{j,s}_{q,l(q)} \leq \sum_{i \in [n]} \sum_{\substack{\theta \in \Theta(i):\\ c_{j,i}+\xi_{j,i}^s \leq \theta}} b_{q,i,\theta} & \quad \forall j \in [N],\ s \in [S],\ q \in Q\label{eq:mpGlob:con5}\\
    &z^{j,s}_{q,r(q)} \leq \sum_{i \in [n]} \sum_{\substack{\theta \in \Theta(i):\\ c_{j,i}+\xi_{j,i}^s > \theta}} b_{q,i,\theta} & \forall j \in [N],\ s \in [S],\ q \in Q\label{eq:mpGlob:con6}\\
    & \bar{z}^{j,s} + M^{j} (1- z^{j,s}_{p(k),k}) \geq \bm{c}_{j}^\top \bm{x}_{k} &  \forall j \in [N],\ s \in [S],\  k \in K \label{eq:mpGlob:con7}\\
    & u \geq \sum_{j \in [N]} \bar{z}^{j,s} & \forall  s \in [S] \label{eq:mpGlob:con8}\\
    & b_{q,i,\theta} \in \{ 0,1 \} \quad & \forall q \in Q,\ i \in [n],\ \theta \in \Theta(i) \label{eq:mp:dom1}\\
    & z^{j,s}_{q,c} \in \{ 0,1 \} \quad & \forall j \in [N],\ s\in[S],\ q \in Q,\ c\in\{l(q),r(q)\}\label{eq:mp:dom2}\\
    & \bar{z}^{j,s} \in \mathbb{R} \quad & \forall j \in [N],\ s\in[S]\label{eq:mp:dom4}\\
        & u \in \mathbb{R} & \label{eq:mp:dom5}
\end{align}
\label{eq:mpGlob}
\end{subequations}

Constraint~\eqref{eq:mpGlob:con2} guarantees that each of the $K$ leaves contains a feasible solution to the underlying optimization problem.
For every inner node, Constraint~\eqref{eq:mpGlob:con1} ensures that exactly one attribute is selected for the split, along with a corresponding threshold value. For each sample point in every scenario, one unit of flow is introduced at the root of the decision tree and is directed towards one of its successor nodes via Constraint~\eqref{eq:mpGlob:con3} and flow conservation at the remaining inner nodes is enforced via Constraint~\eqref{eq:mpGlob:con4}.
Constraints~\eqref{eq:mpGlob:con5} and \eqref{eq:mpGlob:con6} ensure that for each perturbed observation of a data point, the flow adheres to the decision tree's structure based on the selected branching decisions. Hence, the flow indicates which perturbed observation is assigned to which solution. Based on this, the resulting objective for every disturbed observation is calculated via Constraint~\eqref{eq:mpGlob:con7}.
Finally, Constraint~\eqref{eq:mpGlob:con8} completes the epigraph formulation of the objective function by taking into account the worst-case sum of costs. Sufficiently large values for the big $M$ parameters in Constraint~\eqref{eq:mpGlob:con7} can be computed by setting $M^j = \max_{\bm{x} \in \mathcal{X}} \bm{c}_{j}^\top \bm{x}$ for all $j \in [N]$ or any other upper bound for the objective value of the solution for an observation $j \in [N]$.

When implementing the master problem, it is possible to reduce both the number of constraints and variables. Instead of creating $N \cdot S$ copies of each variable to represent flow through the tree, we can aggregate identical disturbed observations of a sample point into a single variable, avoiding $S$ copies for every sample point even when identical disturbances occur in the uncertainty set. Thus, for each sample point $j \in [N]$, we only need as many copies of the flow variables as there are distinct perturbations $\bm{\xi}^s_j$ for $s\in [S]$. This approach requires maintaining a data structure to track which disturbed observations correspond to each scenario $s \in [S]$.

\medskip

After constructing the decision tree defined by the split queries, represented by variables $b_{q,i,\theta}$, and the solutions at the leaves $\bm{x}_k$, the adversarial problem is invoked to determine the disturbance in the observed attributes that leads to the worst-case overall costs when applying the optimization surrogate to the sample points. To formulate the adversary problem, we can exploit the fact that we can precompute the efforts the adversary has to undertake to disturb observations of sample point $j\in [N]$ such that the surrogate assigns it to the solution provided in leaf $k\in K$. In particular, let $\rho_k^j$ represent the minimum sum of the absolute changes in the observed attributes $\bm{c}_j$ that are required to yield the perturbed observation $\hat{\bm{c}}_j$, which would be mapped to leaf $k$ in the decision tree derived from the master problem. In particular, assume item $i\in[n]$ is queried at node $q\in Q$ along the path to node $k$ with threshold $\theta \in \Theta(i)$. If $l(q)$ leads to leaf $k$ then the cost of $\max\{0,c_{i}^j-\theta\}$ is added to $\rho_k^j$. If $r(q)$ leads to leaf $k$ then $\max\{0,\theta-c_{i}^j+\epsilon\}$ is added, for some small $\epsilon>0$. By definition, we have $\rho^k_j=0$ for the leaf $k$ that is reached in observation $j$ by the decision tree without disturbance. Note that these perturbation costs are the same for both types of uncertainty sets. In order to find the worst-case perturbation, we can formulate the adversary problem as a multiple choice knapsack problem \cite{sinha1979multiple}, which is known to be NP-hard \cite{kellerer2004multiple}. In case of the global budgeted uncertainty set, Model~\eqref{eq:apGlob} can be used, where variable $y_k^j$ indicates whether the observation of sample point $j\in[N]$ is disturbed in such a way that it is mapped to leaf $k\in K$.

\begin{subequations}
\begin{align}
    \max & \sum_{j \in [N]} \sum_{k \in K} (\bm{c}_{j}^\top \bm{x}_{k}) y^{j}_{k}& \\
    \text{s.t.} & \sum_{j \in [N]} \sum_{k \in K} \rho^{j}_{k} y^{j}_{k} \leq \Gamma^{\text{glob}} & \label{eq::exchange}\\
    & \sum_{k \in K} y^{j}_{k} = 1 & \forall j \in [N] \\
    & y^{j}_{k} \in \{ 0,1 \} & \forall j \in [N], k \in K
\end{align}
\label{eq:apGlob}
\end{subequations}
For the case of the local budgeted uncertainty, the adversary problem is obtained by exchanging Constraint~\eqref{eq::exchange} with the constraint 
\begin{equation}
\sum_{k \in K} \rho^{j}_{k} y^{j}_{k} \leq \Gamma^\text{loc}\quad \forall j \in [N] .
\end{equation}
In this case, the problem can be decomposed so that for each $j\in[N]$, we solve
\begin{align*}
    \max & \sum_{k \in K} (\bm{c}_{j}^\top \bm{x}_{k}) y^{j}_{k}& \\
    \text{s.t.} & \sum_{k \in K} \rho^{j}_{k} y^{j}_{k} \leq \Gamma^\text{loc} \\
    & \sum_{k \in K} y^{j}_{k} = 1 \\
    & y^{j}_{k} \in \{ 0,1 \} & \forall k \in K.
\end{align*}
Each  subproblem can be solved to optimality in polynomial time by simply selecting $k\in\arg\max_{k\in K}\{ \bm{c}_{j}^\top \bm{x}_{k} : \rho^{j}_{k} \le \Gamma^\text{loc}\}$.

Further note that for both types of uncertainty sets, a worst-case assignment of samples to leaves---given by variables $y_k^j$---can be transformed into a corresponding perturbation $\bm{\xi}_j$ of the observation of sample $j\in[N]$, by backtracking the path from $k$ to the root and setting $\xi_{j,i}$ according to the cost that are incurred in $\rho_k^j$ for this sample.  If the optimal objective value of the adversary problem is equal to the objective value of the master problem, the process terminates, as no scenario exists that can increase the worst-case cost. If, on the other hand, the objective value of the adversary problem exceeds the cost of the master problem, the newly found scenario is added to $\mathcal{U}'$ and we return to solve the upgraded master problem.

To further improve a decision tree $\tree$ found after solving Problem~\eqref{eq:mpGlob}, we propose a post-processing step. For every node $q \in Q$ and the corresponding variable $b_{q,i,\theta}$ for which $b_{q,i,\theta} =1$ in the solution of Problem~\eqref{eq:mpGlob}, let $\bar{c}_{j,i} < \theta < \bar{c}_{j+1,i}$. We then consider the interval $[\bar{c}_{j,i}, \bar{c}_{j+1,i}]$ this threshold is part of and discretize it by $\Theta^\prime(q)=\{\pi \bar{c}_{j,i}+(1-\pi)\bar{c}_{j+1,i} \mid \pi \in \Pi\}$ with discrete $\Pi \subset(0,1]$.
For every $(\theta^\prime_1,\ldots, \theta^\prime_{|Q|}) \in \Theta^\prime(1) \times \ldots \times \Theta^\prime(|Q|)$, i.e., for every constellation of thresholds that are in the same interval as the thresholds of $\tree$, we create a new tree that  uses the new thresholds. These trees are evaluated solving Problem~\eqref{eq:apGlob}. We return the tree with the minimal objective value found. In this paper, we use $\Pi = \{0.1, 0.2, \ldots, 0.9\}$.

\section{Heuristics} \label{sec:heuristics}
We present one baseline heuristic as well as three more sophisticated heuristics for generating robust decision trees. All heuristics share the idea of dividing the problem of finding a tree into two (interdependent) parts: finding a good tree structure (i.e. branching attributes and corresponding thresholds) and determining the solutions associated with its leaves. We are trying to reduce the computation times by treating the variables of one of these as parameters and only optimizing the other subproblem.

\subsection{Baseline Heuristic \OneSol: Single Solution\label{sec::1Sol}}
One obvious way of simplifying Problem~\eqref{eq::TheProblem} is to restrict the allowed structure of the decision tree even further than necessary for comprehensibility. In our baseline approach---referred to as \OneSol---we restrict ourselves to only considering trees with a depth of zero. In particular, we do not need to find any splits as the allowed trees only consist of a single (leaf) node, i.e., every data point is mapped to the same solution $\bm{x}$. Furthermore, no uncertainty has to be taken into account (as there are no splits). The resulting problem of finding the optimal solution can be formulated as
$$\min_{\bm{x}\in \X} \sum_{j \in [N]}  \bm{c}_j^\top \bm{x}.$$

\subsection{Heuristic \HTree: Fix Tree, Optimize Solution}\label{subs:heuristics:fixtree}

The approach of this heuristic is to repeatedly generate a random decision tree and then assign the best possible solutions to the leaves. After the time limit is reached, the best combination of tree structure and solutions is returned. Depending on the type of budget to be considered, the procedures for determining the solutions differ. We refer to this heuristic as \HTreeGlob and \HTreeLoc for the global and local uncertainty set, respectively.

Consider a given tree structure.
In case of a global budget, an iterative approach is needed to find the best solutions $\bm{x}_k$ for leaves $k\in K$. We start out with a subset of perturbations, which again can be assumed to be $\mathcal{U}' = \{(\bm{0},\ldots,\bm{0})\}$. Let $[S]$ be the enumeration of this set and let $s\in[S]$ refer to one specific perturbation. As the tree structure is fixed, for every sample $j\in [N]$, we can precompute the leaf it is mapped to. Hence, the binary parameters $y_k^{j,s}$ indicate whether sample $j \in [N]$ is mapped to leaf $k \in [K]$ given perturbation $s \in [S]$.  By solving the reduced Master Problem~\eqref{eqs:globFixTreeMp} we obtain optimal solutions for the given combination of tree structure and perturbations. \begin{subequations}
\begin{align}
    \min \quad & u \\
    \text{s.t.} \quad & u \geq \sum_{j \in [N]} \sum_{k \in K}
     (\bm{c}_{j}^\top \bm{x}_{k}) y^{j,s}_{k} & \forall s \in [S] \\
    & \bm{x}_{k} \in \mathcal{X} & \forall k \in K\\
    & u \in \mathbb{R} & 
\end{align}
\label{eqs:globFixTreeMp}
\end{subequations}
 In the first iteration, the solutions are optimized on the nominal setting and the resulting decision is not necessarily well-protected against perturbed observations. As in Section~\ref{sec:optApproach} we now need to detect the most harmful perturbation of the observations by solving the Adversary Problem~\eqref{eq:apGlob}.
The found perturbation is added to $\mathcal{U}'$, increasing $S$ by one, and Problem~\eqref{eqs:globFixTreeMp} is solved again. By using this iterative procedure until the objective function values are identical, we can find  solutions that are optimal for the fixed tree structure.

When dealing with a local budget, there is no need for an iterative approach, as the worst-case perturbation of every sample point can be considered independently. In particular, for every sample $j\in[N]$ we can again calculate the adversary cost $\rho_k^j$ of perturbing the observations in such a way that it ends up in leaf $k$ (see Section~\ref{sec:iterative}). Then, we can extract $\bar{K}(j)=\{k\in K : \rho_k^j\leq \Gamma^{\text{loc}} \}$, which is the set of leaves the adversary can send sample $j$ to. Then,  Problem~\eqref{eqs:locFixTree} finds the optimal solutions  to be
assigned to the leaves of the fixed tree.
\begin{subequations}
\begin{align}
    \min \quad & \sum_{j \in [N]} u_j \\
    \text{s.t.} \quad & u_j \geq \bm{c}_{j}^\top \bm{x}_{k} & \forall j \in [N], k \in \bar{K}(j) \\
    & \bm{x}_k \in \mathcal{X} & \forall k \in K \\
    & u_j \in \mathbb{R} & \forall j \in [N]
\end{align}
\label{eqs:locFixTree}
\end{subequations}
The entire procedure is outlined in Algorithm~\ref{alg:fixTreeHeu}, differentiating between the local and the global budgeted uncertainty set with  $\tree^\star$ being the current incumbent best surrogate.
\begin{algorithm}
\begin{algorithmic}[1]
\State $\tree^* \gets$ None, $\textnormal{obj}(\tree^*)\gets \infty$
\While{time limit is not reached\label{line:fixTreeHeu:loop}}
    \ForAll{$q \in Q$}\Comment{Sample tree structure} \label{line:treestartfirstloop}
    \State $b_{q,i,\theta} \gets 0 \ \forall i \in [n]\ \forall \theta \in \Theta(i)$ \label{line:fixTreeHeu:splitsEnd}
        \State $i^\prime \gets$ sample uniformly from $[n]$ \label{line:fixTreeHeu:splitsStart}
        \State $\theta^\prime \gets$ sample uniformly from $\Theta(i^\prime)$
        \State $b_{q,i^\prime,\theta^\prime} \gets 1$
    \EndFor
    \State for all $j\in[N]$, $k\in K$: compute $\rho_k^j$ \label{line:treeendfirstloop}
    \If{$\Uloc$}\Comment{Determine solutions $\bm{x}_k$}
    \State for all $j\in[N]$: compute $\bar{K}(j)$ \label{line:treestartloc}
        \State solve Problem~\eqref{eqs:locFixTree}; extract surrogate $\tree$ and objective value $\textnormal{obj}^{\textnormal{mas}}$ \label{line:treeendloc}
    \ElsIf{$\Uglob$}
        \State $\mathcal{U}' \gets \{(\bm{0},\ldots,\bm{0})\}$; $S\gets 1$ \label{line:treestartglob}
        \State converged $\gets$ False
        \While{\textbf{not} converged}
        \State
        for all $k\in K$, $j\in[N]$, $s\in S$: compute $y_k^{j,s}$ 
        \State solve Master Problem~\eqref{eqs:globFixTreeMp}; extract surrogate $\tree$ and objective value $\textnormal{obj}^{\textnormal{mas}}$
         \State solve Adversary Problem~\eqref{eq:apGlob}; extract perturbation $\bm\xi$ and objective $\textnormal{obj}^{\textnormal{adv}}$
            \If{$\textnormal{obj}^{\textnormal{adv}} = \textnormal{obj}^{\textnormal{mas}}$}
                \State converged $\gets$ True
                \Else
                            \State $\mathcal{U}' \gets \mathcal{U}' \cup \{\bm\xi\}$; $S\gets S+1$ \label{line:treeendglob}
            \EndIf
        \EndWhile
    \EndIf
    \If{$\textnormal{obj}^{\textnormal{mas}} < \textnormal{obj}(\tree^*)$}
    \State $\textnormal{obj}(\tree^*)\gets\textnormal{obj}^{\textnormal{mas}}$
        \State $\tree^* \gets \tree$
    \EndIf
\EndWhile
\State \Return $\tree^*$
\end{algorithmic}
\caption{\label{alg:fixTreeHeu} \HTree heuristic}
\end{algorithm}

Within the outer loop that runs until a time limit is met, we first sample a random tree structure in Lines~\ref{line:treestartfirstloop}--\ref{line:treeendfirstloop}. We then determine solutions $\bm{x}_k$ for each leaf $k\in K$, using either Lines~\ref{line:treestartloc}--\ref{line:treeendloc} in case of local budgeted uncertainty, or Lines~\ref{line:treestartglob}--\ref{line:treeendglob} in case of global budgeted uncertainty, which requires an inner loop. Finally, we keep track of the best solution found in the algorithm and return this best solution at the end.

\subsection{Heuristic \HSol: Fix Solutions, Optimize Tree}\label{subs:heuristics:fixsol}
The second heuristic samples a set of $K$ solutions and randomly assigns each of them to one of the leaves. The remaining task is to find splits of the tree, that perform best if these solutions are fixed at the leaves. This procedure again is performed repeatedly and after the time limit is reached, the best combination of solutions and tree structure is returned. To find the best tree for one fixed set of solutions, we reuse the iterative approach presented in Section~\ref{sec:iterative}, where we now can fix the values $\bm{x}_k$, $k\in K$, according to the sampled solutions. In this case, we need to apply the iterative approach both in the cases of global and local budgeted uncertainty set, obtaining heuristics \HSolGlob and \HSolLoc, respectively.

\begin{algorithm}
\begin{algorithmic}[1]
\State $\tree^* \gets$ None, $\textnormal{obj}(\tree^*) \gets \infty$, $X^{cand} \gets \emptyset$
\ForAll{$j \in [N]$} \Comment{Determine solution pool}\label{line:1}
    \State $X^{cand} \gets X^{cand} \cup \{\arg \min_{\bm{x} \in \mathcal{X}} \bm{c}_j^\top \bm{x}\}$ \label{line:2}
\EndFor
\While{time limit is not reached \label{line:fixSolsHeu:loop}}
\ForAll{$k\in K$} \Comment{Sample random solutions} \label{line:3}
    \State $\bm{x}_k \gets$ sample uniformly from $X^{cand}$
    \label{line:fixSolsHeu:fix} \label{line:4}
\EndFor
\State  $\mathcal{U}' \gets \{(\bm{0},\ldots,\bm{0})\}$; $S\gets 1$ \label{line:5}
    \State converged $\gets$ False \Comment{Determine tree structure}
    \While{\textbf{not} converged}
        \State solve Master Problem~\eqref{eq:mpGlob} with $\bm{x}_k$ fixed
        \State extract surrogate $\tree$ and objective value $\textnormal{obj}^{\textnormal{mas}}$
         \State solve Adversary Problem~\eqref{eq:apGlob}; extract perturbation $\bm\xi$ and objective $\textnormal{obj}^{\textnormal{adv}}$
         
          \If{$\textnormal{obj}^{\textnormal{adv}} = \textnormal{obj}^{\textnormal{mas}}$}
                \State converged $\gets$ True
                \Else
                            \State $\mathcal{U}' \gets \mathcal{U}' \cup \{\bm\xi\}$; $S\gets S+1$
            \EndIf
        \EndWhile \label{line:6}
    \If{$\textnormal{obj}^{\textnormal{mas}} < \textnormal{obj}(\tree^*)$}
    \State $\textnormal{obj}(\tree^*)\gets\textnormal{obj}^{\textnormal{mas}}$
        \State $\tree^* \gets \tree$
    \EndIf
\EndWhile
\State \Return $\tree^*$

\end{algorithmic}
\caption{\label{alg:fixSolsHeu} \HSol heuristic.}
\end{algorithm}

We first generate a pool of solution candidates in Lines~\ref{line:1}--\ref{line:2}. In this case, we simply choose solutions that optimize with respect to one of the historical cost vectors; in principle, other methods to generate suitable candidate solutions can be applied here as well. We then repeat the remaining process until the time limit is reached. We first sample a random subset of solutions in Lines~\ref{line:3}--\ref{line:4}, and then optimize for the tree structure that takes uncertainty into account using the scenario generation loop in Lines~\ref{line:5}--\ref{line:6}. The last lines of the algorithm are used to keep track of the best solution found.

\subsection{Heuristic \HAlt: Alternate \HTree and \HSol}\label{subs:heuristics:alternating}

Our final heuristic approach is based on the idea of merging heuristics \HTree and \HSol in an alternating fashion. We first call \HTree once, i.e.\ only one iteration of the loop in Line~\ref{line:fixTreeHeu:loop} of Algorithm~\ref{alg:fixTreeHeu} is performed. This provides us with a random tree structure and the respective optimized solutions at its leaves.

We now take these optimized solutions and use them in \HSol as if they were randomly sampled. In particular, we fix these solutions and optimize the splits given the fixed solutions, i.e., we perform a single iteration of the loop starting in Line~\ref{line:fixSolsHeu:loop} of Algorithm~\ref{alg:fixSolsHeu} where we do not sample random solutions $\bm{x}_k$, but instead fix them to the solutions found in \HTree.

Now, we extract the splits from the newly found surrogate from \HSol and call \HTree once again, where we replace Lines~\ref{line:treestartfirstloop}--\ref{line:treeendfirstloop} of Algorithm~\ref{alg:fixTreeHeu} such that the splits are set according to these extracted splits. Again, the obtained solutions that are optimal for this tree are used in \HSol, and this process is repeated until both methods generate the same objective value.

Similar to the other two heuristics, this is performed several times, until a time limit is reached, after which the best encountered surrogate is returned. Figure~\ref{fig:enter-label} represents the nested structure of one run of this heuristic, referred to as \HAltGlob for the case of a global budgeted uncertainty set. Recall that for the case of a local budgeted uncertainty set, \HAltLoc does not have to perform an iterative approach in order to conduct the \HTreeLoc sub routine.
\begin{figure}[h!]
    \centering
        \begin{tikzpicture}[node/.style={circle,draw,font=\sffamily\Large\bfseries}]
            \node[draw, shape=rectangle, rounded corners, line width=0.3mm] at (-3,0) (1) {master problem};
            \node[above = -.02cm of 1,xshift=-2.2cm]  (HTree) {\HTreeGlob};
            \node[draw, shape=rectangle, rounded corners, line width=0.3mm] at (-3,-2) (2) {adversary problem};
            \node[draw, shape=rectangle, rounded corners, line width=0.3mm] at (3,0) (3) {master problem};
   \node[above = -.02cm of 3,xshift=2.2cm]  (HTree) {\HSolGlob};
            \node[draw, shape=rectangle, rounded corners, line width=0.3mm] at (3,-2) (4) {adversary problem};
            \node[draw, dashed, shape=rectangle, rounded corners, line width=0.3mm, minimum height=4cm, minimum width=5.5cm] at (-3,-1) (5) {};
            \node[draw, dashed, shape=rectangle, rounded corners, line width=0.3mm, minimum height=4cm, minimum width=5.5cm] at (3,-1) (6) {};
            \path[every node/.style={font=\sffamily\small}]
                    (1.360) edge[->, bend left=45,text width=2.4cm, align=center, line width=0.3mm] node [right, rotate=90, anchor=north] {solutions} (2.360)
                    (2.180) edge[->, bend left=45,text width=2.4cm, align=center, line width=0.3mm] node [left, rotate=90, anchor=south] {perturbation} (1.180)
                    (3.360) edge[->, bend left=45,text width=2.4cm, align=center, line width=0.3mm] node [right, rotate=90, anchor=north] {tree} (4.360)
                    (4.180) edge[->, bend left=45,text width=2.4cm, align=center, line width=0.3mm] node [left, rotate=90, anchor=south] {perturbation} (3.180)
                    (5.90) edge[->, bend left=25,text width=2.4cm, line width=0.3mm, align=center] node [above, anchor = south] {solutions} (6.90)
                    (6.270) edge[->, bend left=25,text width=2.4cm, line width=0.3mm, align=center] node [below, anchor = north] {tree} (5.270);
        \end{tikzpicture}
    \caption{Scheme of the alternating heuristic for a global budget (\HAltGlob).}
    \label{fig:enter-label}
\end{figure}
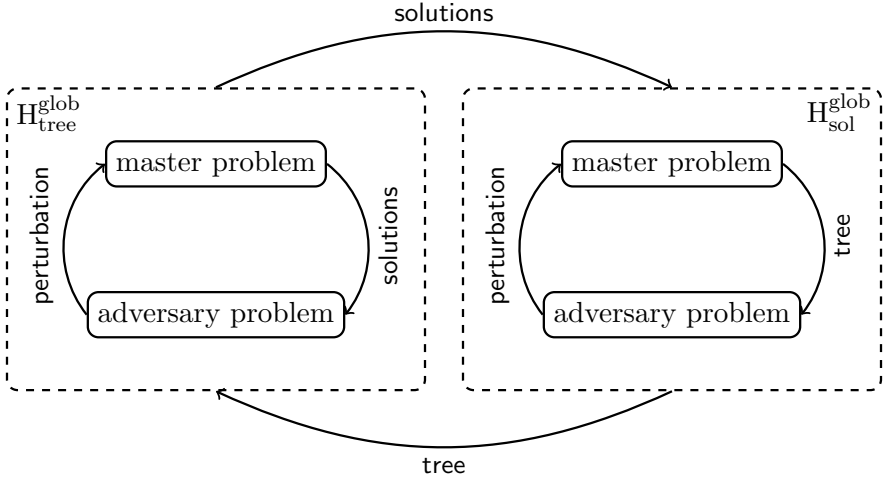

\section{Complexity}\label{sec:complexity}

In the previous section, we introduced the heuristic \OneSol, which simply maps every cost vector to the same, single solution. As the next theorem shows, this approach is even optimal if the uncertainty budget is sufficiently large, which means that the adversarial becomes powerful enough to map any cost vector to any leaf of a decision tree.

\begin{theorem}
Let $\Trees$ be the set of  univariate trees of depth $D\in\mathbb{N}$ as described in Section~\ref{subs:optApproach:uncertainty}. Let $M=\max_{i\in[n]}\left\lbrace \max_{j\in[N]} c_{j,i}-\min_{j\in[N]}c_{j,i}\right\rbrace$.
For a global budgeted uncertainty set with $\Gamma^{\text{glob}}> D\cdot N \cdot M$ and for a local budgeted uncertainty set with $\Gamma^{\text{loc}}> D\cdot M$, the solution obtained by \OneSol is optimal for Problem~\eqref{eq::TheProblem}.

\label{thm::OneSol}
\end{theorem}

\begin{proof}
For a fixed adversary decision on the perturbation of observations, any sample that is fed to the decision tree has to pass $D$ queries. Hence, any adversary decision results in the perturbation of at most $D\cdot N$ observations. Regardless of the threshold for the splits, it never makes sense to perturb the observations further than the range of the queried item in the data. Consequently, for any split, a deviation of at most $M=\max_{i\in[n]}\left\lbrace \max_{j\in[N]} c_{j,i}-\min_{j\in[N]}c_{j,i}\right\rbrace$ of the observed value is implemented by the adversary.
Let us consider the cases $\Gamma^{\text{glob}}>D\cdot N\cdot M$ and $\Gamma^{\text{loc}}>D\cdot M$. In this case, as argued above, the adversary is capable of perturbing the observations in such a way, that any sample can end up in any of the leaves. Hence, the selection of splits no longer affects the result; only the solutions chosen at the leaves matter. Then
\begin{align*}
\min_{\tree \in \Trees} \max_{(\bm{\xi}_1,\ldots, \bm{\xi}_N) \in \mathcal{U}} \sum_{j \in [N]} \bm{c}_j^\top \tree(\bm{c}_j+\bm{\xi}_j) &=  \min_{\bm{x}_1,\ldots,\bm{x}_{|K|}\in \X} \sum_{j \in [N]} \max_{k \in K} \bm{c}_j^\top \bm{x}_k\\
&\geq  \min_{\bm{x}_1,\ldots,\bm{x}_{|K|}\in \X} \max_{k \in K} \sum_{j \in [N]}  \bm{c}_j^\top \bm{x}_k\\
&=  \min_{\bm{x}\in \X} \sum_{j \in [N]}  \bm{c}_j^\top \bm{x}
\end{align*}
and therefore, \OneSol, which surely yields a feasible surrogate, is optimal for these cases.

\end{proof}

As a consequence, the complexity of Problem~\eqref{eq::TheProblem} becomes the same as minimizing a linear function over $\X$.

\begin{cor}
In a situation as outlined in Theorem~\ref{thm::OneSol}, Problem~\eqref{eq::TheProblem} has the same complexity as the optimization problem specified by the underlying domain $\mathcal{X}$.
\end{cor}

As a second consequence, any other heuristic that optimizes the solutions of the leaves of a decision tree, will also be optimal for sufficiently large budget.

\begin{cor}
In a situation as outlined in Theorem~\ref{thm::OneSol}, heuristics \HTree and \HAlt obtain an optimal solution to Problem~\eqref{eq::TheProblem}.
\label{cor::HTreeHAlt}
\end{cor}

We also show that our proposed heuristic of fixing the splits of a decision tree and then optimizing over the solutions assigned to the leaves needs to solve an NP-hard problem for both types of uncertainty. To this end, we define the problem (\LSAP) of selecting the best solutions for a decision tree under uncertainty.
\begin{definition}{The Leaf Solution Assignment Problem (\LSAP)}~\\
Given: $N$ historical samples of cost vectors
$\bm{c}_1,\ldots,\bm{c}_N$,
a set of feasible solutions $\mathcal{X}\subseteq \mathbb{R}^n$, a univariate tree of depth $D$ with fixed splits and thresholds and with $K=2^D$ leaves, and an uncertainty set that perturbs observations  of the samples.

Task: find a set of solutions associated to each leaf $\bm{x}_1,\ldots,\bm{x}_{|K|}$ that minimizes the average cost with respect to the worst-case mapping of samples to leaves, resulting from perturbations of the observations by the  adversary.
\end{definition}
In contrast to Problem~\eqref{eq::TheProblem}, the splits of the tree are fixed and only the  solutions need to be selected. This problem definition covers both the local and global budgeted uncertainty set.

We show that this problem is NP-hard already for the selection problem, where $\X=\{\bm{x}\in\{0,1\}^n : \sum_{i\in[n]} x_i = p\}$ for some given integer $p$. To that end, we construct a reduction from the partition problem, which is known to be NP-complete \cite{garey1979computers}.
\begin{theorem}
The Problem~(\LSAP) is NP-hard for the selection problem, even if $\Gamma=1$.
\end{theorem}
\begin{proof}
We start out with the partition problem with $n$ items, $n\in\mathbb{N}$ even, and a positive integer weight $w_i$ associated to every item $i \in [n]$. Let $W=\sum_{i \in [n]}w_i\geq 2$. We need to decide if there exists a subset $P\subseteq [n]$ with $p= \vert P\vert =\frac{n}{2}$, such that $\sum_{i \in P}w_i=\frac{W}{2}$.

The arising instance of (\LSAP) builds upon the selection problem with $\mathcal{X}=\{\bm{x}\in\{0,1\}^{n+p} : \sum_{i\in[n+p]}x_i=p\}$. There are $N=3$ samples and a decision tree of depth $D=1$. The single split in the decision tree queries the cost of item $n+1$ with threshold $W$. The historical samples are defined as shown in Table~\ref{tab::reduction}, with $M\geq 3pW$ being a sufficiently large value.
\begin{table}[htb]
\centering
\caption{Samples of the cost vector\label{tab::reduction}}
\begin{tabular}{c|ccc|cccc}
item &$1$&$\cdots$ &$n$&$n+1$&$n+2$&$\cdots$ &$n+p$
\\\hline
$\bm{c}_1$ & $W-w_1$&$\cdots$&$W-w_n$&$M$&$M$&$\cdots$ &$M$\\
$\bm{c}_2$ & $M$&$\cdots$&$M$&$0$&$0$&$\cdots$ &$0$\\
$\bm{c}_3$ & $2w_1$&$\cdots$&$2w_n$&$W$&0&$\cdots$ &$0$\\
\end{tabular}
\end{table} 
Both for the global and the local budgeted uncertainty set, we set the budget to $1$. Hence, samples $\bm{c}_1$ and $\bm{c}_2$ are not affected by the adversary's perturbation and always end up in different leaves, which we name leaf $1$ and leaf $2$, respectively. 
The adversary can perturb the observation of item $n+1$ of 
$\bm{c}_3$, causing sample 3 to end up in either leaf. Let $\mu(1)=\{1\}$, $\mu(2)=\{2\}$ and $\mu(3)=\{1,2\}$ describe these mapping options of the adversary, which are valid for both types of uncertainty.

Due to the structure of $\bm{c}_1$, which is assigned to leaf $1$ in the decision tree, the solution $\bm{x}_1$ only selects from the first $n$ items.  Furthermore, using the same argument, in $\bm{x}_2$ none of the first $n$ items are selected, but instead all items $n+1$,\ldots, $n+p$ are chosen.

Let $I \in [n]$ be the set of items chosen in solution $\bm{x}_1$ and let $X(I)= \sum_{i\in I} w_i$. The costs with respect to sample $\bm{c}_1$ are always $pW-X(I)$, the costs with respect to sample $\bm{c}_2$ are always $0$, and for sample $\bm{c}_3$ they depend on the adversary's decision given by $\mu(3)$: If the first leaf is selected by the adversary, $\bm{x}_1$ is implemented resulting in costs of $2X(I)$. If the second leaf is selected the costs are $W$. Hence,

\begin{align*}
\min_{\bm{x}_1,\ldots,\bm{x}_{|K|} \in \mathcal{X}} \sum_{j\in[N]} \max_{k\in \mu(j)} \bm{c}_j^\top\bm{x}_k &= \min_{I\in[n]} \max\{pW-X(I)+2X(I),pW-X(I)+W\}\\
&= \min_{I\in[n]} pW+ \max\{X(I),W-X(I)\}\, .
\end{align*}
By selecting $I\in [n]$ such that $X(I)=\frac{W}{2}$   these costs are minimal. In particular, if and only if there is a solution to the partition instance, the optimal costs of the instance of the (\LSAP) instance are equal to $(p+\frac{1}{2})W$.
\end{proof}

Finally, note that for a budget of zero, regardless of the uncertainty set, the (\LSAP) problem decomposes to $K$ separate nominal problems $\min_{\bm{x}\in \mathcal{X}} \bar{\bm{c}}_k^\top\bm{x}$,
where $\bar{\bm{c}}_k$ is the sum of all scenarios that are assigned to leaf $k$ by the decision tree.
Hence, in this case (\LSAP) is in P, if the underlying optimization problem is in P.

\section{Computational Experiments}\label{sec:experiments}
\subsection{Setup}\label{subs:experiments:setup}

In this section, computational experiments testing our approaches are presented. In particular, we want to answer the three following sets of questions:
\begin{enumerate}
    \item Is there a correlation in the performance of a decision tree on both uncertainty sets presented in Section \ref{subs:optApproach:uncertainty}?
    \item To which extent do our methods generate more robust decision trees than methods which only take a nominal training environment into account? What is the benefit of relying on a (potentially vulnerable) tree structure vs. using only one solution?
    \item What is the trade-off between nominal and robust objective value? How well do the trees generated generalize to new data? How does the instance size influence the performance of our methods?
\end{enumerate}

For each experiment, an indicated subset of the methods presented in Sections \ref{sec:optApproach} and \ref{sec:heuristics} will be used. We will benchmark against decision trees which were trained following the approach from \cite{goerigk2023framework} using only nominal observations and therefore do not consider uncertainty. Instead of using the authors' MIP formulation, we use Model~\eqref{eq:mpGlob} considering only the unperturbed observations resulting in optimal decision trees for the nominal setting. We will refer to this method as the nominal one and use the representative \textit{\NomTree} in plots and tables. By using the nominal approach and \OneSol in every experiment as benchmarks, we cover both extreme cases: not considering uncertainty at all and hedging in the most conservative way possible (see Theorem~\ref{thm::OneSol}). 
Besides our heuristics \HTree, \HSol, and \HAlt, we also test the exact iterative solution approach, which is denoted as SG.

We perform tests on synthetic data for shortest path problems. The artificial instances consist of $n \times n$ grid graphs, in which all edges are directed from south to north and west to east, respectively. The objective is to find the shortest path from the node in the southwestern corner to the node in the northeastern corner. The costs of the edges are generated by using a variation of the procedure of \cite{goerigk2023framework}. Three basis scenarios were created, consisting of an individual range of possible cost values for every edge. To generate one observation, first, one of these basis scenarios was selected randomly. Then, for every edge a value representing its cost was sampled uniformly from the interval corresponding to this combination of basis scenario and edge.
 
We will refer to the $N$ samples as training data, and the 1000 data points used for evaluation as test data. The quality of solutions evaluated using the former is also referred to as in-sample performance, using the test data as out-of-sample performance. Furthermore, we will differentiate between the nominal performance, where we evaluate using only the unperturbed observations, and the robust performance where we evaluate using the observations perturbed in the most harmful way possible w.r.t.~to the indicated uncertainty set. Every data point in the following figures and tables is based on the averaged results of 20 instances. All 20 instances are different, but the same 20 instances were used for each data point.

For our experimental tests, both uncertainty sets that we introduced were considered. Instead of stating the absolute values of $\Gamma^{loc}$ and $\Gamma^{glob}$, we will report the parameter $\lambda \in [0,1]$ which is used to compute both. For a given instance, $\Gamma^{loc}$ can be computed by using
$$\Gamma^{loc} = \lambda D \max_{i\in[n]}\left\lbrace \max_{j\in[N]} c_{j,i}-\min_{j\in[N]}c_{j,i}\right\rbrace,$$
i.e., $\lambda > 1$ implies that for a given tree with a depth of at most $D$, every data point in our training data can be perturbed such that it could end up in every leaf (see Theorem~\ref{thm::OneSol}). In contrast, $\lambda = 0$ allows for no perturbation at all and is therefore equal to a setting with no uncertainty. The specific value used for $\Gamma^{glob}$ will we be written in dependency of $\Gamma^{loc}$.

The numerical experiments were set up using Python 3.11.2. Graphs were constructed and managed by using the networkx library \cite{networkx}. Furthermore, we used Gurobi \cite{gurobi} version 11.0 and its Python interface to solve the MIP formulations. For each combination of instance and method, a time limit of one hour was set. All experiments were conducted on two virtual machines with 8 cores running on 2.4 GHz and 12 GB memory each. Gurobi's core usage was limited to one and up to eight instances were solved in parallel per machine. For $\epsilon$, which was used during the pre- and post-processing of the adversary problem, a value of 0.001 was chosen to avoid numerical issues. The code and the generated data are available on GitHub\footnote{\url{https://github.com/sbstnmrtn/robust_interpretable_surrogates}}.

\subsection{Results}\label{subs:experiments:reults}
\subsubsection{First Experiment}
Two types of uncertainty sets are introduced in Section~\ref{subs:optApproach:uncertainty}. We want to examine if the performance of a decision tree in both settings is correlated. A strong correlation could imply that it is only necessary to protect against one of these uncertainty sets to perform well in both environments. 

To this end, we sample random decision trees and evaluate their robustness using both uncertainty sets. The possible outcomes for both uncertainty sets depend on the given budget. However, this has a different meaning in the two sets, which makes a meaningful comparison difficult. We therefore use two different settings throughout this experiment. First, we investigate the case where we set $\Gamma^{glob} = N\Gamma^{loc}$ which implies $\Uloc \subseteq \Uglob$. Those results are illustrated in Figures~\ref{fig::corr::005} to \ref{fig::corr::020}. Second, we examine the setting where $\Gamma^{glob} = \Gamma^{loc}$ and hence $\Uglob \subseteq \Uloc$, shown in Figures~\ref{fig::corr::005E} to \ref{fig::corr::020E}.
\begin{figure}[htb]
    \begin{subfigure}[t]{0.24\textwidth}
         \centering
         \includegraphics[width=1\linewidth]{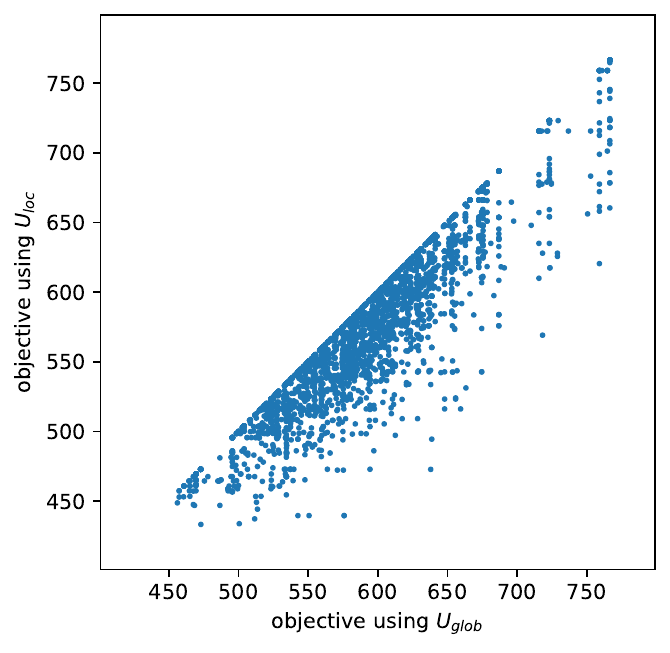}
         \caption{$\Gamma^{glob}= N\Gamma^{loc},\\ \lambda=0.05, r=0.92$.\label{fig::corr::005}}
    \end{subfigure}
    \begin{subfigure}[t]{0.24\textwidth}
         \centering
         \includegraphics[width=1\linewidth]{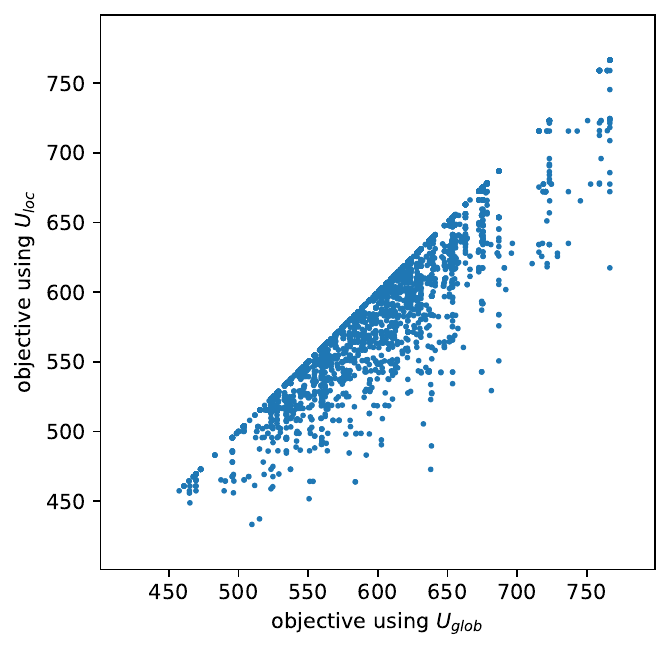}
         \caption{$\Gamma^{glob}= N\Gamma^{loc},\\ \lambda=0.10, r=0.92$. \label{fig::corr::010}}
    \end{subfigure}
    \begin{subfigure}[t]{0.24\textwidth}
         \centering
         \includegraphics[width=1\linewidth]{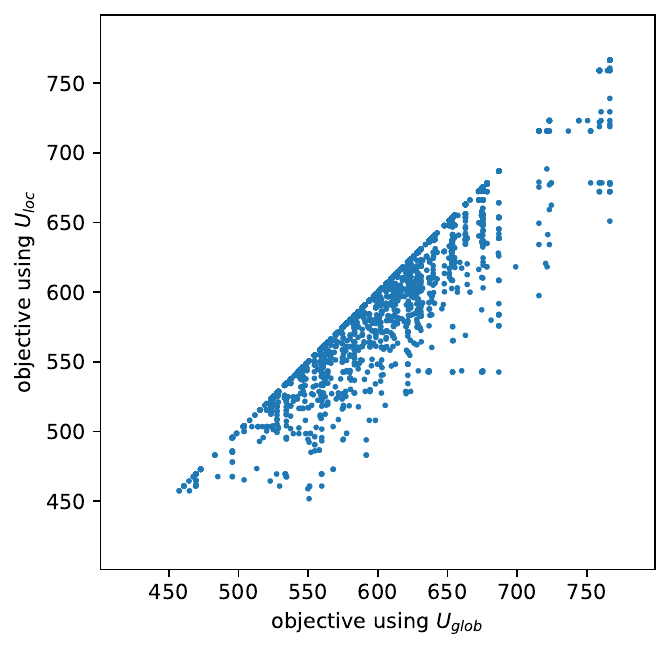}
         \caption{$\Gamma^{glob}= N\Gamma^{loc},\\ \lambda=0.15, r=0.94$.}
         \label{fig::corr::015}
    \end{subfigure}
    \begin{subfigure}[t]{0.24\textwidth}
         \centering
         \includegraphics[width=1\linewidth]{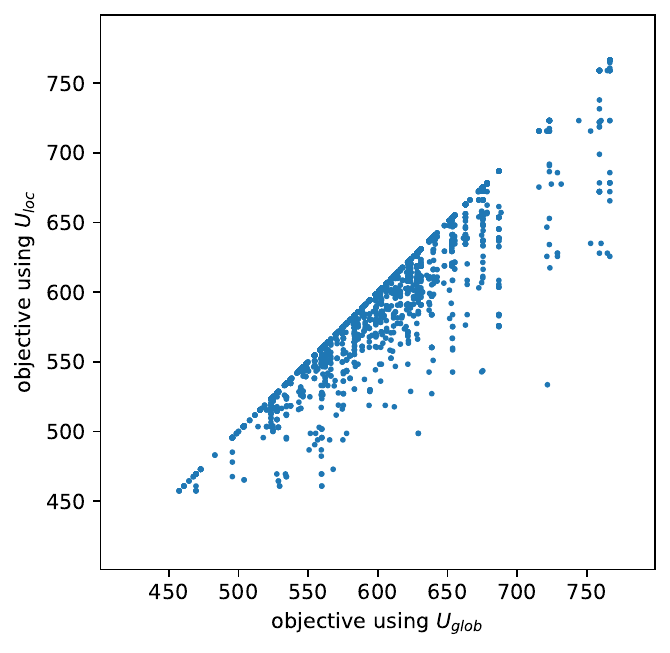}
         \caption{$\Gamma^{glob}= N\Gamma^{loc},\\ \lambda=0.20, r=0.95$.}
         \label{fig::corr::020}
    \end{subfigure}
    \begin{subfigure}[t]{0.24\textwidth}
         \centering
         \includegraphics[width=1\linewidth]{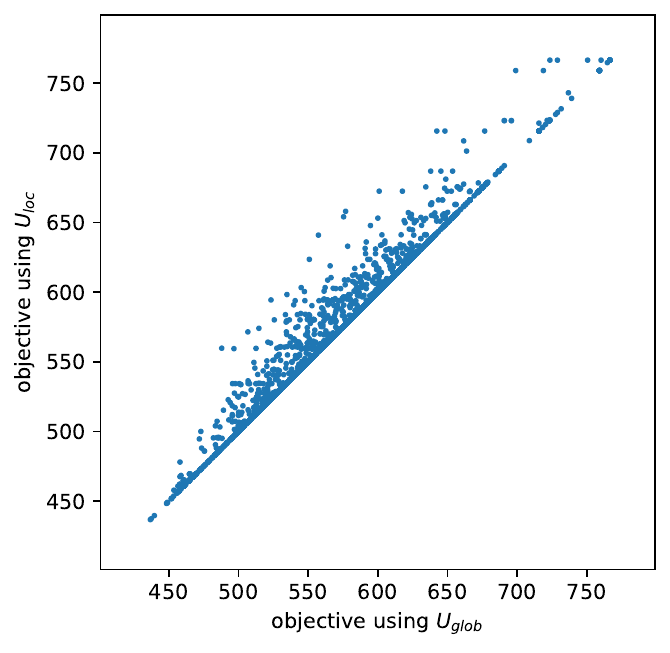}
         \caption{$\Gamma^{glob}= \Gamma^{loc},\\ \lambda=0.05, r=0.99$.}
         \label{fig::corr::005E}
    \end{subfigure}
    \begin{subfigure}[t]{0.24\textwidth}
         \centering
         \includegraphics[width=1\linewidth]{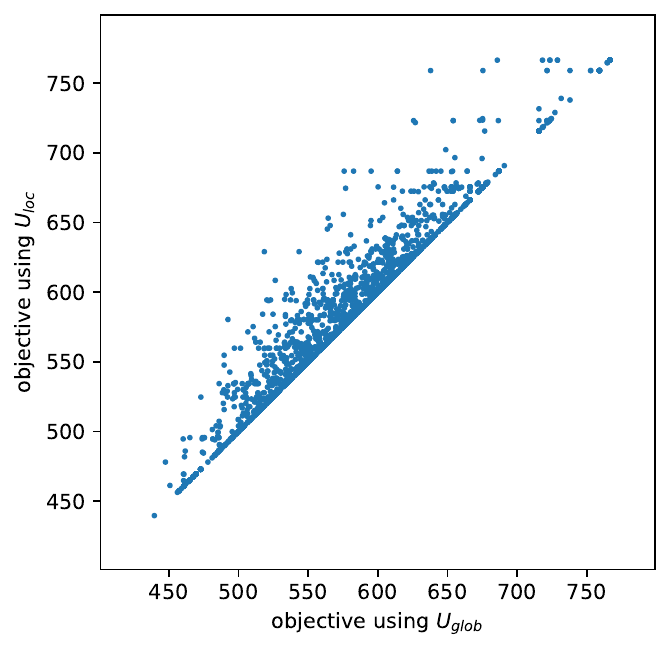}
         \caption{$\Gamma^{glob}= \Gamma^{loc},\\ \lambda=0.10, r=0.97$.}
         \label{fig::corr::010E}
    \end{subfigure}
    \begin{subfigure}[t]{0.24\textwidth}
         \centering
         \includegraphics[width=1\linewidth]{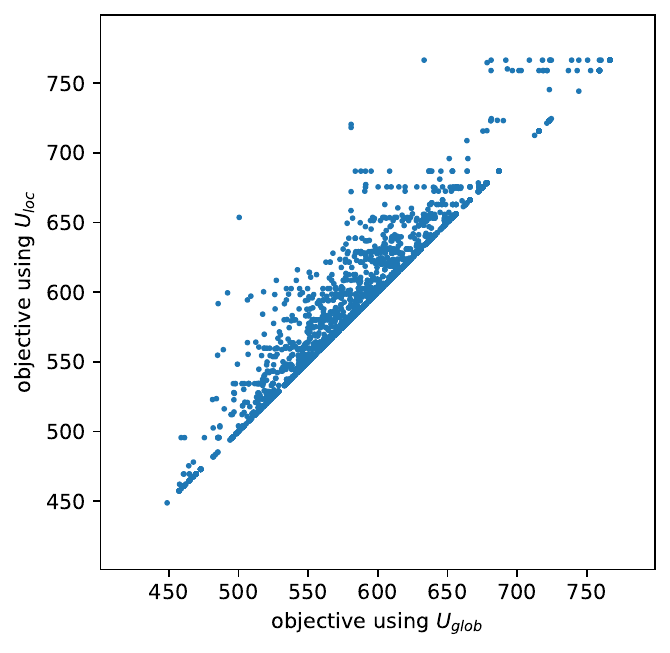}
         \caption{$\Gamma^{glob}= \Gamma^{loc},\\ \lambda=0.15, r=0.97$.}
         \label{fig::corr::015E}
    \end{subfigure}
    \begin{subfigure}[t]{0.24\textwidth}
         \centering
         \includegraphics[width=1\linewidth]{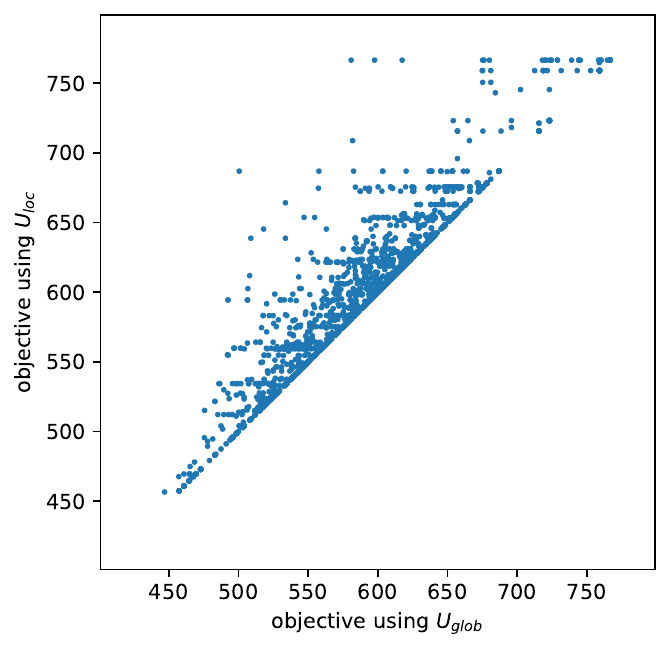}
         \caption{$\Gamma^{glob}= \Gamma^{loc},\\ \lambda=0.20, r=0.96$.}
         \label{fig::corr::020E}
    \end{subfigure}
    \caption{Plots experiment 1.    \label{fig::corr}}

\end{figure}

For every plot we used 20 cost matrices, but the same 20 for every plot, each consisting of $N=5$ observations on a $4\times4$ grid graph and sampled 200  surrogates.
Each plot therefore contains 4,000 data points. The robust in-sample performance was determined for each surrogate by solving Problem~\eqref{eq:apGlob} with respect to the given tree structure and both uncertainty sets. Each surrogate is represented by one point, where the horizontal coordinate indicates its performance w.r.t. $\Uglob$, and its vertical coordinate indicates its performance on $\Uloc$. The Pearson correlation coefficients are presented as $r$ in the captions.

It can be seen that there is a strong (linear) correlation for all tested settings. For the case where $\Gamma^{glob} = N\Gamma^{loc}$ an increase in correlation for increasing values of $\lamloc$ can be determined.
An explanation for this is
that for a growing portion of the trees, an increase in $\lambda$ does not lead to an increase in objective value w.r.t. to $\Uglob$. As stated in Theorem~\ref{thm::OneSol}, setting $\lamloc$ to a value greater that one
for $\Gamma^{glob} = N\Gamma^{loc}$ allows the adversary to perturb the dataset such that every observation can end up in every leaf given any tree as described in Section~\ref{subs:optApproach:uncertainty}. For one specific decision tree, eventually a value of $\lamloc$ of less than one can be sufficient to already allow to perturb the dataset in such a way. Since $\Uloc(\Gamma^{loc}) \subseteq \Uglob(N\Gamma^{loc})$, the values of $\lambda$ where this effect can be seen w.r.t. $\Uglob(N\Gamma^{loc})$ are therefore smaller or equal than those for $\Uloc(\Gamma^{loc})$. Thus, the increasing correlation can be explained by the effect that an increase in $\lambda$ raises the objective value using $\Uloc$ and therefore brings it closer to its upper limit given by objective using $\Uglob$.

In contrast, in the setting where $\Gamma^{glob} = \Gamma^{loc}$ the correlation decreases with rising $\lambda$. Here the increase in $\lambda$ allows for disproportionately worse perturbations for $\Uglob$ than for $\Uloc$. This effect will invert after a specific value of $\lambda$, analogously to the case described before.

\subsubsection{Second Experiment}
In this experiment, the methods presented in Sections~\ref{sec:optApproach} and \ref{sec:heuristics} are compared with respect to their robust in-sample performance using both uncertainty sets. Furthermore, the influence of different values of $\lambda$ is investigated. For all results discussed in the context of this experiment, instances with $n = 4$ and $N=5$ were used. We have set $\Gamma^{glob}= N\Gamma^{loc}$.

Figures~\ref{fig::gamma::glob} and \ref{fig::gamma::loc} illustrate the robust in-sample objective value of the best surrogate which could be found within the time limit using each method evaluated on the global and local budgeted uncertainty set, respectively. On the horizontal axis the value of $\lamloc$ which was used for training as well the evaluation is indicated. Solid lines are utilized to indicate methods which were trained with respect to the uncertainty set used for evaluation. Dotted lines specify methods optimized for the respective other uncertainty set. An increment of 0.01 was selected for the interval $\lambda \in [0,0.1]$. For $\lambda \in [0.1,0.2]$ an increment of 0.02 was chosen.

\begin{figure}[htb]
    \centering
         \includegraphics[width=.8\linewidth]{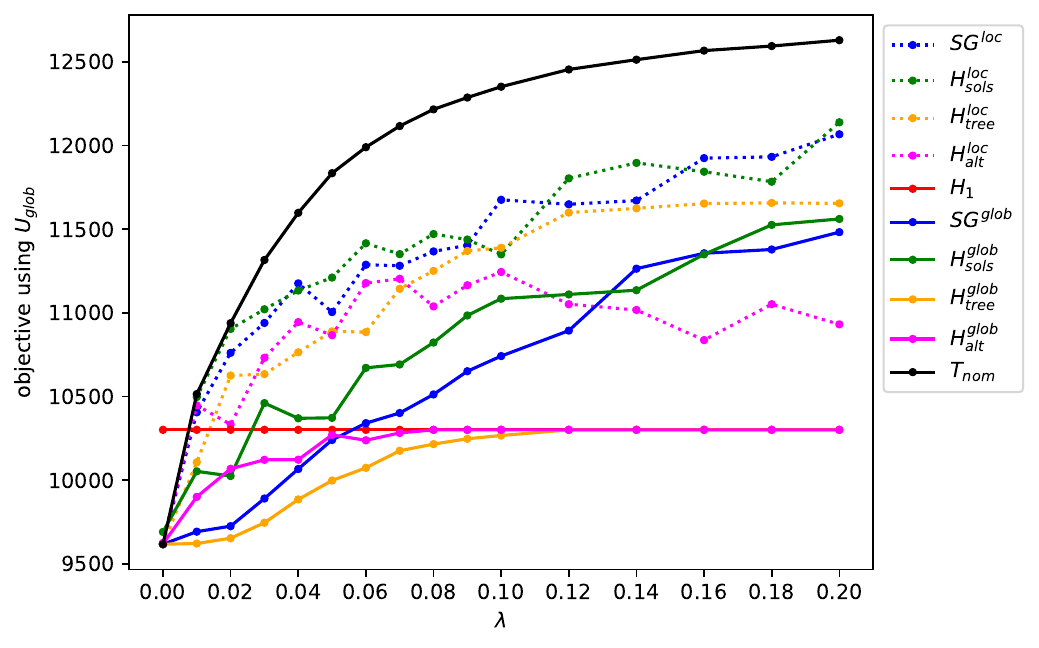}
         \caption{Comparison of solution methods evaluated using robust objective values in-sample with $\Uglob$.}
         \label{fig::gamma::glob}
    \end{figure}
    
\begin{figure}[htb]
         \centering
         \includegraphics[width=.8\linewidth]{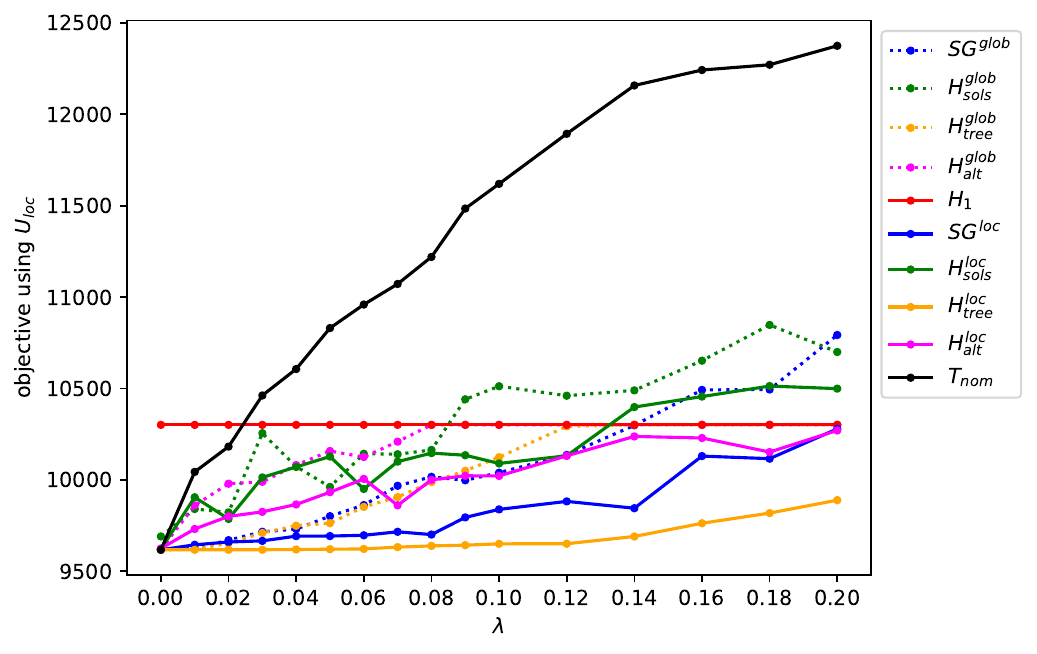}
         \caption{Comparison of solution methods evaluated using robust objective values in-sample with $\Uloc$.}
         \label{fig::gamma::loc}
\end{figure}

It can be seen that when using the global budgeted uncertainty set for evaluation,
it is beneficial to use either \HTreeGlob or \HAltGlob
for generating robust decision trees when $\lambda \in [0.01, 0.1]$.
For $\lamloc=0$, i.e. a setting without uncertainty, it is advantageous to use the nominal approach for constructing a tree. In the heuristics case it is due to the fact that they are not guaranteed to find an optimal solution. Furthermore, we can not benefit from using our scenario generation approach, which would terminate after the first iteration. For values of $\lamloc > 0.1$ (i.e. a setting with a high degree of uncertainty), it is best to use \OneSol, \HAltGlob or \HTreeGlob.

Figure~\ref{fig::gamma::loc} shows that the observations also show rising objective values for an increasing value of $\lambda$ in the setting using $\Uloc$. In contrast to the previous results, the increase here is less strong. This is congruent with the observations in the first experiment and our theoretical considerations in Section \ref{subs:optApproach:uncertainty}. As shown in Theorem \ref{thm::OneSol} at least for a $\lambda$ value of one, the use of any of our methods will result in a solution that only is as good as \OneSol in the best case. In fact, we would expect for a subset of our instances and some $\lambda < 1$ to not be able to find a better solution than \OneSol (as observed in Figure \ref{fig::gamma::glob} using $\Uglob$). For this setting, however, our experiments did not cover the point at which the usage of \OneSol would be advantageous.

Despite the strong correlation between the performance of a tree on both uncertainty sets, it becomes obvious that it is better to hedge against the specific uncertainty set. Still, in nearly every case, it is better to have optimized using the respective other uncertainty set compared to not considering uncertainty at all with \NomTree.

In terms of comparison between our methods, \HTree unambiguously outperforms all other methods with respect to both uncertainty sets. It can be observed that the heuristics \HAlt and \HTree indeed find surrogates which are always better or equal to \OneSol (see Corollary \ref{cor::HTreeHAlt}). \HSol performs for most instances noticeably worse than the other robust methods used. For relatively small values of $\lamloc$, good solutions can be found using the iterative solution approach.
The surprisingly poor performance of this (potentially optimal) approach for larger values of $\lambda$ can be explained by the fact that in most cases it was not possible to find optimal solutions within the time limit of one hour. Table \ref{tab:solvedopt} shows the number of instances in dependency of $\lambda$ and the uncertainty set for which solutions with an optimality gap smaller or equal to $0.001$ could be found using this approach. It can be seen that the problems become harder to solve with an increasing value of $\lambda$. Furthermore, problems using $\Uglob$ tend to be harder to solve than problems using $\Uloc$ for the same $\lambda$.

\begin{table}[htb]
\centering
    \caption{Number of instances solved to optimality by \SG. \label{tab:solvedopt}}
   
\begin{tabular}{c|rrrrrrrrrrr}
 $\lambda$    & 0.00 & 0.01 & 0.02 & 0.03 & 0.04 & 0.05 & 0.06 & 0.07 & 0.08 & 0.09 & $\geq$0.10\\
     \hline
    $\Uloc$ &  20 & 12 & 10 & 14 & 6 & 8 & 6 & 5 & 7 & 3 & 11 \\
    $\Uglob$ & 20 & 8  &  2 &  1 & 1 & 0 & 0 & 0 & 0 & 0 & 0 
\end{tabular}
\end{table}

Since \HTree was observed to perform best, for further experiments we will focus on this method in combination with both benchmark approaches.

\subsubsection{Third Experiment}
The experiments described in this section
were conducted varying the number of scenarios~($N$) used during applying our methods as well as the grid size of the underlying shortest path instances~($n$). We consider $\lambda:=0.05,\ \Gamma^{glob}:= N\Gamma^{loc}$ and make use of only \HTree, \OneSol and the nominal approach. The absolute budget which was used for the training was also used for the evaluation on the test data.

The results are presented in Table \ref{tab:grid_train::nom} and \ref{tab:grid_train::rob}.
All results are scaled to the objective value of the nominal approach with the respective uncertainty set using the formula  $$\textnormal{obj}^{scaled} =\frac{\textnormal{obj}-\textnormal{obj}^{nom}}{\textnormal{obj}^{nom}}.$$
That is, positive scaled values indicate that the nominal solution performs better, while negative scaled values mean that the comparison method performs better.
In Table \ref{tab:grid_train::nom} the scaled nominal objective is presented. In contrast, Table \ref{tab:grid_train::rob} illustrates the scaled robust objective value.

In Table \ref{tab:grid_train::rob}, we therefore distinguish between \OneSolGlob and \OneSolLoc, even though the absolute objective value of the solution of \OneSol does not depend on the uncertainty set used.

\begin{table}[ht]
    \centering
    \caption{Relative nominal objective value (\%).}
    \label{tab:grid_train::nom}
\begin{tabular}{lll|rrr}
&&& \multicolumn{3}{|c}{Method} \\
 & $N$ & $n$ & \OneSol & \HTreeGlob & \HTreeLoc \\
\hline
\multirow{8}{*}{\rotatebox[origin=c]{90}{Training}} & 3 & 4 & 6.7056 & 0.7510 & 0.0000 \\
& 5 & 4 & 7.1063 & 1.2516 & 0.0019 \\
& 7 & 4 & 8.2096 & 2.0466 & 0.0751 \\
& 10& 4 & 9.0149 & 2.0417 & 0.2371 \\
\cline{2-6}
& 5 & 3 & 6.7006 & 1.2590 & 0.1231 \\
& 5 & 4 & 7.1063 & 1.2516 & 0.0019 \\
& 5 & 5 & 8.4293 & 1.2371 & 0.0043 \\
& 5 & 6 & 8.9419 & 1.3581 & 0.0017 \\
\hline
\multirow{8}{*}{\rotatebox[origin=c]{90}{Test}} & 3 & 4 & 2.3120 & -2.0299 & -2.2002 \\
& 5 & 4 & 4.2236 & -1.7767 & -2.2085 \\
& 7 & 4 & 7.3464 & 0.3625 & -2.0039 \\
& 10 & 4 & 6.0712 & -0.9286 & -3.2824 \\
\cline{2-6}
& 5 & 3 & 4.0298 & -1.6097 & -1.5254 \\
& 5 & 4 & 4.2236 & -1.7767 & -2.2085 \\
& 5 & 5 & 5.1832 & -0.7755 & -2.6882 \\
& 5 & 6 & 3.7592 & -3.1228 & -3.6670 \\
\end{tabular}

\end{table}

Table \ref{tab:grid_train::nom} shows that the loss of \HTree in nominal performance on the training data is genuinely small and varies between zero and about two percent. Especially, when trained with respect to $\Uloc$ it reveals only a slightly worse performance in this setting.

The results furthermore suggest that while our methods perform only slightly worse than \NomTree on the training data in a nominal setting, they tend to outperform \NomTree on unseen test data, even in a setting with no uncertainty. This effect occurs in particular if $\Uloc$ is considered during the training, and seems to intensify with increasing instance size. Moreover, \HTree achieves significantly better performance than \OneSol with respect to both measures.

Table \ref{tab:grid_train::rob} displays the robust performance. Here, \HTree shows clearly better results than \OneSol and \NomTree. In both cases, the relative difference in performance between \HTree and \NomTree is bigger on the training than on the test data. It is striking that \HTreeGlob tends to perform better than \HTreeLoc on the training data, but this effect is reversed when the test data is considered. This can be explained by the use of the same absolute values for $\Gamma^{glob}$ in the training and test environment, which reduces its relative impact in the latter case. For \HTreeGlob the relative robust in-sample objective value with an increasing number of training scenarios seems to increase. No clear trends can be identified for this method in relation to the other tests. In contrast, for \HTreeLoc the relative robust out-of-sample objective values tends to decrease with an increasing number of training scenarios used or increased grid size. Furthermore the robust in-sample performance benefits from an increasing grid size.

\begin{table}[ht]
    \centering
    \caption{Relative robust objective value (\%).}
    \label{tab:grid_train::rob}
\begin{tabular}{lll|rr|rr}
&&& \multicolumn{4}{|c}{Method} \\
& $N$ & $n$ & \OneSolGlob & \HTreeGlob & \OneSolLoc & \HTreeLoc \\
\hline
\multirow{8}{*}{\rotatebox[origin=c]{90}{Training}} & 3 & 4 & -10.6319 & -15.5742 & -3.4331 & -9.5016 \\
& 5 & 4 & -12.9597 & -15.5188 & -4.8888 & -11.1712 \\
& 7 & 4 & -11.8572 & -13.6847 & -3.0417 & -10.3208\\
& 10 & 4 & -10.9140 & -12.2569 & -2.7013 & -10.4983 \\
\cline{2-7}
& 5 & 3 & -10.6479 & -12.9628 & -2.2908 & -8.3140 \\
& 5 & 4 & -12.9597 & -15.5188 & -4.8888 & -11.1712 \\
& 5 & 5 & -10.2432 & -13.5515 & -2.5616 & -10.1326 \\
& 5 & 6 & -12.6607 & -15.9659 & -7.8116 & -15.3566 \\
\hline
\multirow{8}{*}{\rotatebox[origin=c]{90}{Test}} & 3 & 4 & 1.5811 & -2.4949 & -2.4344 & -3.6340 \\
& 5 & 4 & 3.0310 & -2.6985 & -2.4306 & -6.4182 \\
& 7 & 4 & 5.9188 & -0.7639 & -0.2548 & -6.8499 \\
& 10 & 4 & 4.1246 & -2.5499 & -1.6606 & -8.0147 \\
\cline{2-7}
& 5 & 3 & 3.0721 & -2.3654 & -1.4542 & -4.6416 \\
& 5 & 4 & 3.0310 & -2.6985 & -2.4306 & -6.4182 \\
& 5 & 5 & 4.1731 & -1.5007 & -1.1265 & -5.6967 \\
& 5 & 6 & 2.2529 & -4.2903 & -4.7133 & -8.2608 \\
\end{tabular}

\end{table}

\section{Conclusion}\label{sec:conclusion}

To ensure that results of mathematical optimization are used in practice by decision makers, an important requirement is that the solution process needs to be transparent. For this purpose, we studied the framework for finding interpretable surrogates as first introduced in \cite{goerigk2023framework}, which uses a decision tree to map instance data (in this case, cost vectors) to solutions.

While this improves the comprehensibility of the optimization process, real-world decision making is further complicated by the presence of uncertainty. Indeed, it is rarely the case that we have completely accurate information available at the point in time when a decision is taken. For this reason, we propose to extend the scope of interpretable surrogates to also include a robust optimization aspect, which protects against worst-case perturbations of the data we observe. We introduced two uncertainty sets for this purpose, which differ regarding how the size of perturbations on historic observations is bounded. Using an iterative solution process, it is possible to find optimal robust decision trees. We note that finding a worst-case perturbation depends on the type of uncertainty set that is applied. However, this approach is limited to small-scale problems due to its high computational effort. For this reason, we introduced heuristics which are based on solving only one aspect of the optimization problem, i.e., we only focus on finding best solutions in the leaves or on finding a decision tree for given leaf solutions. In a theoretical analysis of the problem complexity, we show that the problem simplifies if the perturbation budget is sufficiently large. Furthermore, we show that the problem of determining best leaf solutions for a given decision tree is already NP-hard.

We conducted several computational experiments to evaluate our approach. We first note that the worst-case objective values stemming from both types of uncertainty sets are correlated, which is an indication that easier-to-treat local budgeted uncertainty sets may be a preferred choice in practice. We then compare the performance of our heuristics and the exact approach with the performance of the nominal solution, that ignores uncertainty. Depending on the size of the uncertainty, we find that robust objective values can be considerably reduced using our methods. As robust objective values represent only one side of the coin, we also evaluate the performance of our approach in more detail using in-sample and out-of-sample objective values using nominal and perturbed scenarios. We find that at small costs with respect to nominal performance, it is possible to improve on all other metrics, and in particular, obtain solutions that perform better out-of-sample.

Several interesting challenges for further research arise. While our methods perform well, they do not provide approximation guarantees, and they are slower than the nominal approach. Finding performance guarantees and further improvements regarding their speed could further strengthen this approach. For simplicity, we introduced one budget $\Gamma^{loc}$ for all observations $j\in[N]$. A straight-forward extension is to allow different budgets $\Gamma^{loc}_j$ for different observations, as well as upper bounds on perturbations $\xi_{j,i}$. Indeed, extending the idea of more flexible uncertainty sets further, an interesting problem is how to design data-driven uncertainty sets, which can be used to estimate possible perturbations from historical data.

\addcontentsline{toc}{section}{References}

\newcommand{\etalchar}[1]{$^{#1}$}

\end{document}